\documentclass{article}

\usepackage{amsmath,amsfonts,bm}

\def\eqref#1{equation~\ref{#1}}
\def\1{\bm{1}}

\DeclareMathAlphabet{\mathsfit}{\encodingdefault}{\sfdefault}{m}{sl}
\SetMathAlphabet{\mathsfit}{bold}{\encodingdefault}{\sfdefault}{bx}{n}

\DeclareMathOperator*{\argmin}{arg\,min}

\usepackage{times}
\usepackage{graphicx}
\usepackage{booktabs} % for professional tables
\usepackage[utf8]{inputenc} % allow utf-8 input
\usepackage{amsfonts}       % blackboard math symbols
\usepackage{nicefrac}       % compact symbols for 1/2, etc.
\usepackage{microtype}      % microtypography
\usepackage{epsfig}
\usepackage{amsmath}
\usepackage{amssymb}
\usepackage{xcolor}
\usepackage{color, colortbl}
\usepackage{mathtools}
\usepackage{lipsum}
\usepackage{booktabs}       % professional-quality tables
\usepackage{wrapfig}
\usepackage{caption}
\usepackage{subcaption}
\usepackage{tabularx,ragged2e}
\usepackage[ruled,vlined]{algorithm2e}
\usepackage{stfloats}
\usepackage{bold-extra}
\usepackage{fancyvrb}
\usepackage{amsthm}
\usepackage{array}
\usepackage{wrapfig}
\usepackage{listings}
\usepackage{multirow}

\newcommand{\flt}{\texttt{FLT}}
\newcommand{\fltours}{\texttt{FLT}\,(ours)}
\newcommand{\cfl}{\texttt{CFL}}
\newcommand{\fedavg}{\texttt{FedAvg}}
\newcommand{\ifca}{\texttt{IFCA}}
\newcommand{\fedsem}{\texttt{FedSEM}}
\newcommand{\fedprox}{\texttt{FedProx}}

\usepackage{geometry}
\usepackage[numbers]{natbib}
\usepackage{chngcntr}
\usepackage{fullpage}
\date{}

\newtheorem{theorem}{Theorem}
\counterwithin*{theorem}{section} 
\newtheorem{assumption}{Assumption}
\newtheorem{lemma}{Lemma}

\usepackage{array}
\usepackage{hyperref}       % hyperlinks
\usepackage{url}            % simple URL typesetting
\newcolumntype{C}[1]{>{\centering\let\newline\\\arraybackslash\hspace{0pt}}m{#1}}

\definecolor{LightCyan}{rgb}{0.8,1,1}

\makeatletter
\newcommand{\ccell}[3][]{%
	\kern-\fboxsep
	\if\relax\detokenize{#1}\relax
	\expandafter\@firstoftwo
	\else
	\expandafter\@secondoftwo
	\fi
	{\colorbox{#2}}%
	{\colorbox[#1]{#2}}%
	{#3}\kern-\fboxsep
}
\makeatother
\definecolor{cellgray}{gray}{0.8}
\AtBeginDocument{\setlength{\cmidrulekern}{0.3em}}

\newcommand{\overbar}[1]{\mkern 1.5mu\overline{\mkern-5.5mu#1\mkern-1.5mu}\mkern 1.5mu}

\title{Federated Learning with Taskonomy for Non-IID Data}

\author{Hadi Jamali-Rad\textsuperscript{\rm 1, 3} \quad~Mohammad Abdizadeh\textsuperscript{\rm 2}\quad~Anuj Singh\textsuperscript{\rm 3}\\{\textsuperscript{\rm 1}Shell Global Solutions International B.V., Amsterdam, The Netherlands}\\{\textsuperscript{\rm 2}Myant Inc., Toronto, ON, Canada}\\{\textsuperscript{\rm 3}Delft University of Technology (TU Delft), Delft, The Netherlands} \\ \small{\texttt{hadi.jamali-rad@shell.com}},~
\small{\texttt{mohammad.abdizadeh@myant.ca}},\\
\small{\texttt{h.jamalirad@tudelft.nl}},~
\small{\texttt{a.r.singh@student.tudelft.nl}}
}

\begin{document}

\maketitle

\begin{abstract}
Classical federated learning approaches incur significant performance degradation in the presence of non-IID client data. A possible direction to address this issue is forming clusters of clients with roughly IID data. Most solutions following this direction are iterative and relatively slow, also prone to convergence issues in discovering underlying cluster formations. We introduce federated learning with taskonomy ({\flt}) that generalizes this direction by learning the task-relatedness between clients for more efficient federated aggregation of heterogeneous data. In a one-off process, the server provides the clients with a pretrained (and fine-tunable) encoder to compress their data into a latent representation, and transmit the signature of their data back to the server. The server then learns the task-relatedness among clients via manifold learning, and performs a generalization of federated averaging. {\flt} can flexibly handle a generic client relatedness graph, when there are no explicit clusters of clients, as well as efficiently decompose it into (disjoint) clusters for clustered federated learning. We demonstrate that {\flt} not only outperforms the existing state-of-the-art baselines in non-IID scenarios but also offers improved fairness across clients.\footnote{Our codebase can be found at: \url{https://github.com/hjraad/FLT/}.}
\end{abstract}

% \begin{IEEEkeywords}
% federated learning, non-IID client data.
% \end{IEEEkeywords}

%%%%%%%%% BODY TEXT
%---------------------------------------------------
\section{Introduction}
\label{sec:intro}
\vspace{-0.0cm}
%---------------------------------------------------
Federated learning is a new paradigm that offers significant potential in elevating edge-computing capabilities in modern massive distributed networks. While presenting great potential, federated learning also comes with its own unique challenges in practical settings \citep{konevcny2015federated, konevcny2016federated}. Recent studies focus on systemic heterogeneity \citep{MAL-083}, communication efficiency \citep{mcmahan2017communication, konevcny2016federated, sattler2019robust}, privacy concerns \citep{geyer2017differentially, bagdasaryan2020backdoor} and more recently on fairness \citep{pmlr-v97-mohri19a, Li2020Fair} and robustness across the network of clients \citep{48698, wang2020attack}. A defining characteristic of massive decentralized networks is stochastic heterogeneity of client data; i.e., clients possess non-independent and identically distributed (non-IID) data. Non-IIDness in client data can be due to several factors, including the following two most commonly considered aspects: i) pathological non-IIDness, where different clients can see different target classes; ii) quantity skew, where different clients can have imbalanced number of samples to train on \cite{mcmahan2017communication}. \citep{li2020federated} identifies  statistical heterogeneity as the root cause for tension between fairness and robustness constraints in federated optimization. \citep{mcmahan2017communication, li2018federated} investigate the impact of heterogeneous data distributions on the performance of federated averaging algorithm, {\fedavg} \cite{mcmahan2017communication}, and demonstrate significant performance degradation in non-IID settings. Several avenues have been explored in the literature to tackle the problem of statistical heterogeneity in federated learning settings. Personalized federated learning tackles data heterogeneity by forming personalized models for clients via meta-learning or multi-task learning \citep{smith2017federated, jiang2019improving, NEURIPS2020_24389bfe, li2020federated, li2021model}. Clustered federated learning addresses this problem by iterative (or recursive) assignment of clients to separate clusters based on model or model update comparisons at the server side \citep{NEURIPS2020_e32cc80b, mansour2020three, sattler2020clustered, briggs2020federated, xie2020multi}. The effectiveness of clustering approaches hinges upon the quality of cluster formation through this iterative assignment process. Besides, clustered federated learning approaches are sensitive to initialization, as we will demonstrate later on. More details on the related work will be provided in the next section. 

Inspired by the idea of ``taskonomy'' \citep{zamir2018taskonomy}, we explore the task-relatedness across client data distributions and cast it in the form of a client relatedness graph. This is accomplished in a \emph{one-off} fashion based on contractive encoding of client data followed by manifold learning (with \texttt{UMAP}) at the server side. The proposed approach ({\flt}) can flexibly handle a range of possibilities in incorporating this client relatedness graph for federated averaging. It can be used in generic form as an extension of {\fedavg} for non-IID data when there are no explicit clusters of clients with similar data distributions, or if need be, can be decomposed with hierarchical clustering (\texttt{HC}) to disjoint clusters and transform into clustered federated learning. Our main contributions can be summarized as follows: i) we propose federated learning with taskonomy ({\flt}), which learns the task-relatedness among clients and uses it at the server-side for federated averaging of non-IID data, without requiring any prior knowledge about the data distribution correlations among clients or the number of clusters they belong to; ii) {\flt} can flexibly discover generic client relatedness as well as an underlying clustered formation in non-IID scenarios; iii) we empirically show that {\flt} offers faster convergence compared to existing state-of-the-art baselines; iv) we provide convergence guarantees on {\flt} under common assumptions required for convergence of {\fedavg}; v) we demonstrate that {\flt} outperforms existing state-of-the-art baselines across several (synthetic and realistic) federated learning settings by about $3\%$, $9\%$, $2\%$ and $40\%$, on MNIST, CIFAR10, FEMNIST, and a newly-introduced ``Structured Non-IID FEMNIST'' dataset); vi) finally, we show that {\flt} offers improved fairness (least variance in performance across clients), besides the improved accuracy.

%---------------------------------------------------
\section{Related Work}
\label{sec:related}
\vspace{-0.0cm}
%---------------------------------------------------
\textbf{Federated learning of non-IID data.} An early proposition to handle non-IID data in federated learning was to create a globally shared dataset comprised of a small subset of data from each client \cite{zhao2018federated}. However, validity of this approach has been debated from increased communication costs, privacy, and security perspectives \cite{MAL-083}. \cite{li2018federated} proposed to add a proximal term to the local optimization sub-problems in order to limit the impact of ``variable'' local updates. This problem is phrased as ``client drift'' in \cite{karimireddy2020scaffold} and a set of control variates are communicated between the clients and server to correct for this drift. In a somewhat related context but on different type of data (such as medical), other studies propose changing the global optimization cost-function or weighting the aggregation at the server side to account for dissimilarities among data distributions \cite{yeganeh2020inverse, laguel2020device, zhao2021federated}. 

\textbf{Personalized and multi-task federated learning.} This line of research tackles data heterogeneity by forming personalized models for clients via meta-learning or multi-task learning \cite{jiang2019improving, NEURIPS2020_24389bfe, smith2017federated}. The work presented in \cite{smith2017federated} extends multi-task learning to federated learning; however, it relies on alternating bi-convex optimization which limits its applicability to only convex objective functions. In meta-learning setting \cite{jiang2019improving, NEURIPS2020_24389bfe}, first a single global model is obtained at the server, based on which each client fine-tunes its model. However, this global model would not serve as a good initialization if the client data distributions are considerably different.  

\textbf{Clustered federated learning.} Akin to our high-level approach, two iterative approaches are proposed in \cite{NEURIPS2020_e32cc80b, mansour2020three} in order to assign clients to separate clusters and train sub-models per cluster. However, several rounds of communication are required until the formation of clusters is solidified. More specifically, in \cite{NEURIPS2020_e32cc80b} per iteration the cluster sub-models have to be sent to all the other active clients in that iterations, which is demanding in terms of communication cost. The clustering methodologies proposed in \cite{sattler2020clustered, briggs2020federated, duan2020fedgroup} are based on recursive bi-partitioning with cosine similarity between model updates as metric. Owing to the recursive nature, their compute and communication overhead can become a bottleneck in large-scale settings. A multi-center federated learning approach is proposed in \cite{xie2020multi} where clients are clustered based on their model parameter differences with randomly initialized cluster-level models. We argue that model parameters are just an implicit proxy of client data distributions, whereas our approach directly exploits the client data itself, and distills it into its signature encoding for clustering. Notably, the approach of \cite{xie2020multi} is sensitive to model initialization  and also requires prior knowledge about the number of clusters. The cited approaches are mostly iterative and either slow in convergence or costly from communication perspective. We instead propose a one-shot solution based on client relatedness. In a concurrent work, \cite{Dennis2020hetero} focuses on one-shot federated clustering. After projecting client data onto a selected subspace, the iterative Lloyd's $k$-means clustering \cite{lloyd1982least} is applied and the outcomes is communicated to the server in a one-shot fashion for cluster assignment. Firstly, \cite{Dennis2020hetero} does not study the impact of the proposed clustering on downstream federated learning applications. Besides, linear subspace decomposition is in practice less efficient than the non-linear (and fine-tuned) autoencoders we employ at client side. Moreover, our manifold learning approach with \texttt{UMAP} \cite{mcinnes2018umap} at the server side has the potential to approximate the underlying manifold of data in a more generic fashion as we empirically demonstrate in the following.  

\textbf{Notation.} In the following, we use $\| X \|_l$ to denote norm-$l$ of matrix $X$, and $|\mathcal{X}|$ to denote the cardinality of set $\mathcal{X}$. We denote the set $\{1, \cdots, n\}$ with $[n]$. We refer to the $(i, j)$-th element of matrix $X$ as $X_{i,j}$ and its $i$-th row with $X_i$. $\1_n$ denotes a row vector of size $n$ containing all $1$'s. 

%---------------------------------------------------
\section{Federated Learning with Taskonomy}
\label{sec:clustFL}
\vspace{-0.0cm}
%---------------------------------------------------
In this section, we first formalize the classical federated learning setting and objectives, and then introduce {\flt}.

%---------------------------------------------------
\subsection{Preliminaries: Refresher on {\fedavg}}
\label{ssec:fedavg}
\vspace{-0.0cm}
%---------------------------------------------------
\begin{algorithm}[tb!]
	%\SetKwInput{Init}{Memory}
	\SetKwInput{Require}{Require}
	\SetAlgoLined
	\DontPrintSemicolon
	\SetNoFillComment
	%\Init{current best model $g^{*}$, current buffer $\mathcal{B}$}
	\Require{$\overbar{M}$, $T$, $w^{0}$, $p_m$}
	\For{$t = 0, \cdots, T - 1$}{
		Server selects a subset $\mathcal{S}^t$ of clients ($\overbar{M} = |\mathcal{S}^t|$);\;
		Server sends $w^t$ to all the selected clients.\;
		\For{\textup{each client} $m$}{
			\For{\textup{epoch} $e = 1, \cdots, E$}{
				$w_m \gets w_m - \eta \nabla F_m(w_m)$\;
			}
			$w_m^{t+1} \gets w_m$\;
		}
		Each client $m$ sends $w_m^{t+1}$ to the server.\;
		Server aggregates $w_m$'s and updates $w$:\;
		\qquad \qquad $w^{t+1} \gets \sum_{m = 1}^{\overbar{M}} p_m \, w_m^{t+1} $\;
	}
	\caption{Federated Averaging ({\fedavg})}\label{alg:fedavg}
\end{algorithm}
%\vspace*{-.4cm}
%\end{algorithm}
%
In a classical federated learning setting, we consider $M$ clients (in practice, hundreds) with $n_m$ local data samples and communicating their learning regularly to a central server to reach global consensus about the whole data composed of $N = \sum_{m} n_m$ samples. In most prior work, the goal is to solve\vspace{-0.2cm}
\begin{equation}\vspace{-0.2cm}
\label{classFL}
\min_w f(w) = \sum_{m = 1}^{M} p_m F_m(w),
\end{equation}
where $p_m = \frac{n_m}{N}$ is the fraction of total data client $m$ sees, and thus, $\sum_m p_m = 1$. The local objective $F_m$ is typically defined by the empirical risk over local data $F_m = \frac{1}{n_m} \sum_{j = 1}^{n_m} l_j(w)$. Federated learning is conducted in regular communication rounds (server to clients, and vice versa) and per round $t$ typically a subset of clients $\mathcal{S}^t$ are randomly selected to run stochastic gradient descent (SGD) for a given number of local epochs. This \emph{local updating} mechanism is shown to be more flexible and efficient than mini-batch methods \citep{stich2018local, wang2018cooperative, woodworth2018graph}. Algorithm~\ref{alg:fedavg} summarizes {\fedavg} \cite{mcmahan2017communication}, a pioneering method to solve \eqref{classFL} in a non-convex setting. The protocol is simple: in $T$ consecutive rounds, selected client $m$ runs $E$ epochs of SGD (with learning rate $\eta$) on local data and shares the local model $w_m$ with the server to be averaged among $\overbar{M}$ participating clients. For the sake of simplicity, the client participation fraction $\rho = \overbar{M}/M = |\mathcal{S}^t|/|\mathcal{S}|$ is typically considered to be constant.  

%---------------------------------------------------
\subsection{Mechanics of {\flt}}
\label{ssec:flt}
\vspace{-0.0cm}
%---------------------------------------------------
Majority of the clustered federated learning approaches \citep{NEURIPS2020_e32cc80b, mansour2020three, sattler2020clustered, briggs2020federated, duan2020fedgroup} enforce a hard membership constraint on the clients to form disjoint clusters where every client can belong to only one cluster. In contrast, we allow for an arbitrary symmetric task-relatedness matrix with the possibility to be reordered and relapsed into disjoint clusters. To form clusters, these approaches mostly compare the clients based on their model parameters using a distance metric (such as $L_2$ norm, or cosine similarity), which in practice does not capture the underlying manifold of data in the model parameter space or any other representation space. We instead propose a one-shot method for learning the task-relatedness matrix (coined as {\flt}) that benefits from manifold approximation (metric learning with \texttt{UMAP} \citep{mcinnes2018umap}) at the server side before applying a distance metric. In the following, we delve deeper into the mechanics of {\flt}.

\textbf{Overview.} A high-level sketch of the proposed approach is depicted in Fig.~\ref{fig:architecture}. As can be seen, we consider three abstraction levels: i) data level, where data samples live in $\mathbb{R}^d$; ii) encoder level, where a contractive latent space representation of client data is extracted in an unsupervised fashion (samples are \emph{nonlinearly} projected to $\mathbb{R}^e$); iii) manifold approximation level with \texttt{UMAP}, where samples live in $\mathbb{R}^u$. The encoder is provided by the server to the clients. This allows them to apply one-shot contractive encoding on their local data, followed by $k$-means on the outcome and return the results to the server. At server side, \texttt{UMAP} is applied to approximate the arriving clients embeddings. This is followed by applying a distance threshold to determine client dependencies and form an adjacency matrix or a client (task) relatedness graph. If forming disjoint clusters is of interest, we then use hierarchical clustering \cite{mullner2011modern} to efficiently reorder the adjacency matrix (or corresponding client relatedness graph) into disjoint clusters (see Fig.~\ref{fig:graph_adjacency}).
% %
% \begin{figure*}[t!]
% 	\begin{minipage}{.65\linewidth}
% 		\centering
% 		\includegraphics[width=0.99\textwidth]{./Figures/architecture.png}
% 		\vspace{-0.1cm}
% 		\caption{High-level architecture of {\flt}.}
% 		\vspace{-0.0cm}
% 		\label{fig:architecture}
% 		\vspace{-0.0cm}
% 	\end{minipage}
% 	\begin{minipage}{.35\linewidth}
% 		\centering
% 		\includegraphics[width=0.8\textwidth]{./Figures/graph_adjacency_2.png}
% 		\vspace{-0.1cm}
% 		\caption{Forming $C = 5$ clusters for a network of $M = 25$ clients with hierarchical clustering. a) adjacency matrix and b) corresponding client relatedness graph (both reordered on the right).}
% 		\vspace{-0.0cm}
% 		\label{fig:graph_adjacency}
% 	\end{minipage}	
% 	\vspace{-0.0cm}
% \end{figure*}
% %

\begin{figure*}[t]
	\centering
	\includegraphics[width=0.9\textwidth]{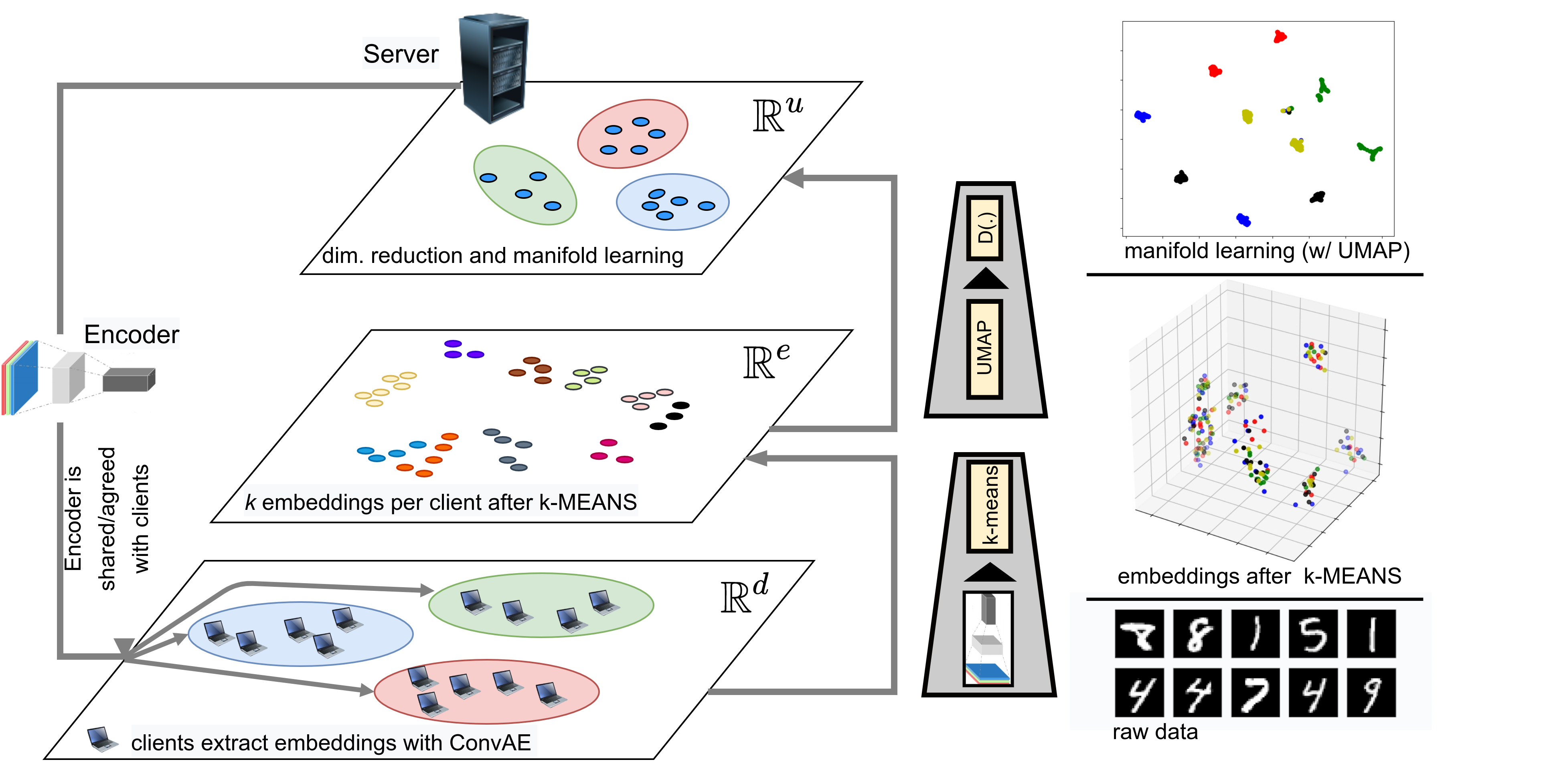}
	\vspace{-0.1cm}
	\caption{High-level architecture of {\flt}. Clients extract their data signature using the agreed ConvAE. They then apply $k$-MEANS to further condense their signatures and transmit $k$ embedding vectors to the server. The server then applies manifold learning using UMAP followed by a distance metric to form an adjacency matrix. Hierarchical clustering (HC) is finally applied if formation of disjoint clusters is of interest as is shown in Fig.~\ref{fig:graph_adjacency}.}
	\vspace{-0.0cm}
	\label{fig:architecture}
	\vspace{-0.0cm}
\end{figure*}

\begin{figure*}[t]
	\centering
	\includegraphics[width=0.8\textwidth]{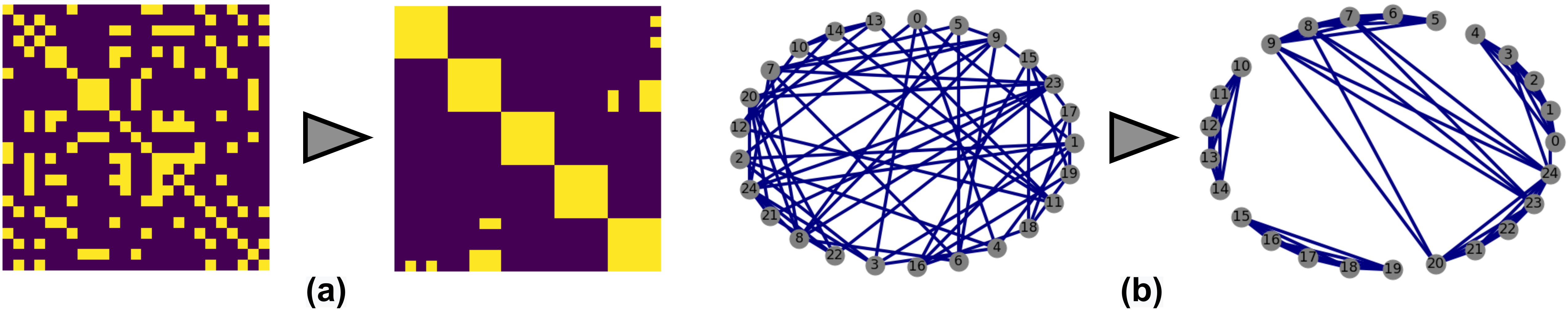}
	\vspace{-0.1cm}
	\caption{Forming $C = 5$ clusters for a network of $M = 25$ clients with hierarchical clustering. a) adjacency matrix and b) corresponding client relatedness graph (both reordered on the right).}
	\vspace{-0.0cm}
	\label{fig:graph_adjacency}
\end{figure*}	

\textbf{Learning client (task) relatedness.} The proposed approach, $\texttt{FCR}$, is described in Algorithm~\ref{alg:formcluster}. The server broadcasts an encoder $G(\cdot)$ to all the $M$ clients in $\mathcal{S}$. Note that this can also be an agreement (between the sever and clients) on using an encoder pretrained on a standard dataset (without virtually sending it), similar to the case of local models being used by clients in {\fedavg}. Nonetheless, this is \emph{one-off}, and this downlink communication/agreement does not have to be repeated unless in an exceptional case where the server decides to change the architecture of the encoder. Note that extracting the signature of client data based on an encoder, instead of model weights for instance, is a key component of {\flt}, and is inspired by taskonomy \cite{zamir2018taskonomy}. While being a significant contributor to the performance of {\flt}, the fact that in some scenarios this encoder has to be fine-tuned on the client data can impose an extra burden. On the other had, since we only consider a \emph{simple} convolutional autoencoder (ConvAE) with its frozen encoder section employed for extracting latent embeddings, this fine-tuning is a straightforward process with negligible added complexity on the client side. In our experimentation (Section~\ref{sec:eval}), this model is as simple as the local client models designed for the downstream task. ConvAE not only helps compressing the information that has to be sent to the server, but also creates a less noisy representation of the client data. Upon receiving $G(\cdot)$ clients compute $\mathcal{E}_m := G(\mathcal{D}_m)$, where $\mathcal{D}_m$ denotes the dataset of client $m$ ($\in [M]$) and $\mathcal{E}_m$ denotes its embedding set of size $|\mathcal{D}_m|$. The elements of $\mathcal{E}_m$ live in $\mathbb{R}^e$ with $e$ referring to the latent embedding dimension. Even though $\mathcal{E}_m$ is compressed as compared to $\mathcal{D}_m$, it turns out that it can still be further distilled and yet capture enough information for our downstream federated learning purpose. Therefore, each client applies $\texttt{kMEANS}(\cdot)$ on $\mathcal{E}_m$ and sends the outcome $\mathcal{M}_m := \{\mu_1, \cdots, \mu_k\}$ (a set of size $k$) back to the server. The \emph{fine-tune} mode is provisioned to accommodate encoders pretrained on a \emph{totally different} dataset. In such a case, clients will be asked to run $F$ epochs of SGD from their latest state on their most recent dataset. As we will demonstrate in Section~\ref{sec:eval}, this is to establish that the choice of dataset for pretraining the encoder is not a bottleneck. 

The server then constructs $\mathcal{M} := \{\mathcal{M}_1, \cdots, \mathcal{M}_M\}$ and applies \texttt{UMAP} \citep{mcinnes2018umap} to $\mathcal{M}$ and computes $\mathcal{Z} := \{\mathcal{Z}_1, \cdots, \mathcal{Z}_{M}\}$ with $\mathcal{Z}_m := \{u_{m,1}, \cdots, u_{m,k}\}$. $\mathcal{Z}$ contains $k \times M$ elements each living in $\mathbb{R}^u$, with $u$ being typically $2$ or $3$. In most prior work, a distance metric ($L_2$ or cosine) is directly applied, which could be a limiting factor for non-convex risk functions and incongruent non-IID settings \citep{sattler2020clustered}. Instead, in $\texttt{FCR}$ the server first learns the manifold in which the embeddings live using \texttt{UMAP} and then applies a distance metric to construct an adjacency matrix $A_{i,j} := \min_{r,s} \|u_{i,r} - u_{j,s}\|_2$, $\forall i, j \in [M]$ \& $\forall r,s \in [k]$, where the minimum pairwise distance among the elements of $\mathcal{Z}_i$ and $\mathcal{Z}_j$ are taken into account. Here, for the sake of simplicity, we consider a \emph{hard-thresholding} operator $\Gamma$ applied on $A$ leading to $\tilde{A}$, where $\tilde{A}_{i,j} = \Gamma(A_{i, j})=\texttt{Sign}(A_{i, j} - \gamma)$ with $\gamma$ a threshold value. In practice, $\gamma$ can be tuned to return best performance in different settings. If constructing \emph{explicit and disjoint} clusters is of interest, the severs applies hierarchical clustering ($\texttt{HC}$) \cite{mullner2011modern} to reorder $\tilde{A}$ into $\tilde{A}^r$ with $C$ disjoint clusters (diagonal blocks). Finally, the server extracts cluster membership in $\mathcal{C} = \{\mathcal{C}_1, \cdots, \mathcal{C}_C\}$ with $\mathcal{C}_c$'s being a set of client ID's in cluster $c$. A flexibility that hierarchical clustering in \cite{mullner2011modern} offers is to propose the best fitting number of clusters, if maximum number of clusters is not specified. $\tilde{A}$ and $\tilde{A}^r$ for a toy setup with $M = 25$ clients are depicted in Fig~\ref{fig:graph_adjacency}: a) adjacency matrix, b) client relatedness graph. On the left, the clients are arranged based on their ID's and on the right they are re-ordered with hierarchical clustering \cite{mullner2011modern} to form $C = 5$ clusters. It is noteworthy that this one-shot client relatedness discovery of $\texttt{FCR}$ can also happen in a few stages, in case all clients are not available for cooperation at the initialization stage. Besides, excluding a few clients from the process (for any reason) will not impact the performance of {\flt}.
\LinesNumbered
\begin{algorithm}[t]
	%\SetKwInput{Init}{Memory}
	\SetKwInput{Require}{Require}
	\SetKwInput{Return}{Return}
	\SetAlgoLined
	\DontPrintSemicolon
	\SetNoFillComment
	%\Init{current best model $g^{*}$, current buffer $\mathcal{B}$}
	\Require{$\texttt{MODE}$, $\mathcal{S}$, $G(\cdot)$, $\Gamma(\cdot)$, $k$, $C$}
	
	%\textbf{if} $\mathcal{P} = \emptyset$ \textbf{then} initialize \textbf{else} fine-tuning.\;  
	Server broadcasts $G(\cdot)$ to all clients in $\mathcal{S}$\;
% 	\eIf{$\mathcal{P} = \emptyset$}{\texttt{MODE} = normal, $\mathcal{P} \gets \mathcal{S}$ \;
% 		Server broadcasts $G(\cdot)$ to all clients $\mathcal{S}$}
% 	{
% 		\texttt{MODE} = fine-tuning
% 	}
% 	%
	\For{\textup{each client $m$ in $\mathcal{S}$}}{
		\If{$\textup{\texttt{MODE}} = \textup{fine-tune}$}{
			Client runs $F$ epochs to fine-tune\;
		}
		Client computes its own embedding:\; 
		\quad $\mathcal{E}_m \gets G(\mathcal{D}_m)$\; 
		Client applies $k$-means clustering to $\mathcal{E}_m$'s:\vspace{+0.04cm}\;  
		\quad $\mathcal{M}_m := \{\mu_1, \cdots, \mu_k\} \gets  \texttt{kMEANS}(\mathcal{E}_m)$\vspace{+0.04cm}\; 
		Client sends $\mathcal{M}_m$ to the server
	}
	Server updates $\mathcal{M} := \{\mathcal{M}_1, \cdots, \mathcal{M}_M\}$\;
	Server (re)computes\; 
	\qquad $\mathcal{Z} := \{\mathcal{Z}_1, \cdots, \mathcal{Z}_{M}\} \gets \texttt{UMAP}(\mathcal{M})$\;
	Server constructs the adjacency matrix:\vspace{+0.04cm}
	$A_{i,j} := \min_{r,s} \|u_{i,r} - u_{j,s}\|_2$,\\ \qquad \qquad \qquad \qquad  $\forall i, j \in [M]$ \& $\forall r,s \in [k]$\vspace{+0.04cm}\; 
	Server applies thresholding $\tilde{A} := \Gamma(A)$\;
	%Server forms disjoint clusters $\tilde{A}^c \gets \texttt{HC}(\tilde{A})$;\;
	Server forms $C$ clusters with hiearchical clustering:\vspace{+0.05cm}\\ 
	\qquad $\mathcal{C} = \{\mathcal{C}_1, \cdots, \mathcal{C}_C\} \gets \texttt{HC}(\tilde{A})$\;\vspace{+0.1cm}
	\Return{$\tilde{A}$ and $\mathcal{C}$}\vspace{+0.1cm}
	\caption{Form Client Relatedness (\texttt{FCR})}\label{alg:formcluster}
\end{algorithm}

\LinesNumbered
\begin{algorithm}[t]
	%\SetKwInput{Init}{Memory}
	\SetKwInput{Require}{Require}
	\SetAlgoLined
	\DontPrintSemicolon
	\SetNoFillComment
	%\Init{current best model $g^{*}$, current buffer $\mathcal{B}$}
	\Require{$\mathcal{S}$, $M$, $T$, $W^{0}$, $p_m$}
	\textbf{Initialize Clustering:}\;
	\qquad $\tilde{A}, \mathcal{C} \gets \texttt{FCR}(\textup{normal}, \mathcal{S}, G(\cdot), \Gamma(\cdot), k)$\;
	\For{$t = 0, \cdots, T - 1$}{
		$w^t_1, \cdots, w^t_m \gets W^t$\;
		Server selects a subset $\mathcal{S}^t$ of clients\;
		% \textup{\textbf{Option II:} dynamic clustering based on $\mathcal{S}^t$:}\;
		% $\tilde{A}^t \gets \FormCluster(\mathcal{S}^t, G(\cdot), U(\cdot), k)$\;
		Server sends $w^t$ to all clients in $\mathcal{S}^t$\;
		\For{\textup{each client} $m$ in $\mathcal{S}^t$}{
			\For{\textup{epoch} $e = 1, \cdots, E$}{
				$w_m \gets w_m - \eta \nabla F_m(w_m)$\;
			}
			$\overbar{w}_m^{t+1} \gets w_m$\;
		}
		Each client sends  $\overbar{w}_m^{t+1}$ and $\delta$ to the server\;
		Server collects\, $\overbar{W}^{t+1} = [\,\overbar{w}^{t+1}_1, \cdots,\, \overbar{w}^{t+1}_M]$\;
% 		\{\textbf{Dynamic Clustering (optional):}\; 
% 		each client sends data change flag $\delta$ \;
% 		Server updates $\Delta = [\delta_1, \cdots, \delta_M]$;\;
% 		\If{$\sum_i{\delta_i} > \lambda$ \textup{\textbf{or}} $(t \bmod \tau = 0)$}{
% 			$\mathcal{P} \gets \textup{supp}(\Delta)$\;
% 			$\tilde{A} \gets \texttt{FCR}(\mathcal{P}, \mathcal{S}, G(\cdot), \Gamma(\cdot), k)$\;
% 			$\Delta \gets [0, \cdots, 0]$\;}\}\;
		Server aggregates and updates\, $\overbar{W}^{t+1}$:\;
		\textbf{Case I}) Full adjacency matrix $\tilde{A}$:\;  
		\quad $W^{t+1} \gets  \overbar{W}^{t+1} \tilde{A} \,\texttt{diag}(p_m / \|\tilde{A}_m\|_0)$\;
		\textbf{Case II}) Disjoint clusters based on $\mathcal{C}$:\;  
		\quad $w^{t+1}_c \gets \frac{1}{|\mathcal{C}_c|}\sum_{m \in \mathcal{C}_c} p_m \, \overbar{w}_m^{t+1}, \, \forall c \in [C]$\;
		\quad $W^{t+1} = [\mathbf{1}_{|\mathcal{C}_1|} w^{t+1}_1, \cdots, \mathbf{1}_{|\mathcal{C}_C|} w^{t+1}_C]$\\
		
        \{\textbf{Dynamic Clustering (optional):}\; 
		Server updates $\Delta = [\delta_1, \cdots,  \delta_{M+M'}]$;\;
		\If{$\sum_i{\delta_i} > \lambda$ \textup{\textbf{or}} $(t \bmod \tau = 0)$}{
			%$\mathcal{P} \gets \textup{supp}(\Delta)$\;
			$\mathcal{S} \gets \mathcal{S} \cup \{M+1, \cdots, M+ M'\}$\;
			$\tilde{A}, \mathcal{C} \gets \texttt{FCR}(\textup{fine-tune}, \mathcal{S}, G(\cdot), \Gamma(\cdot), k)$\;
			$\Delta \gets [0, \cdots, 0]$, $M \gets M+ M'$\;}\}\;
			}
	
	\caption{{\flt}}\label{alg:flt}
\end{algorithm}

\textbf{Clustering performance of {\flt}.} Let us assume that the data of each client is generated from a mixture of $L$ components, where $L$ is the total number of target classes (labels). The goal of the encoder is to transform the data into a latent embedding such that separability amongst data labels is maximized for clustering. Let us denote the classification error of the encoder part of the pretrained ConvAE by $P_e^{\text{cae}}$. For sake of simplicity, assume that latent embeddings are a mixture of Gaussian components. In such a case, the following theorem provides an upper-bound on the total clustering error of {\flt}:
\begin{theorem}
Clustering error of {\flt} is upper-bounded by:
\begin{align}
 P^{tot}_{\textup{err}} \lessapprox \sum_{m=1}^{M} \sum_{l=1}^{L} \sum_{i=1}^{e}  \exp\left(-\frac{\left(n_m^l \times Q^{-1}(P_{\textup{err}}^{\text{cae}}) \right)^2}{2}\right).   
\end{align}
where, $Q$ represents the $Q$-function and, $n_m^l$ denotes the total number of samples with label $l$ for client $m$, and $e$ is the dimension of the latent space .
\end{theorem}
This demonstrates that by increasing the sample size, the total clustering error of {\flt} vanishes, as desired. The detailed proof can be found in Appendix~\ref{ssec:clusteringaccuracy}.

\textbf{Federated averaging with taskonomy.} {\flt} in Algorithm~\ref{alg:flt} starts with an initialization stage by calling the \emph{normal} mode of $\texttt{FCR}$. This returns the adjacency (or client relatedness) matrix $\tilde{A}$ together with an optional cluster membership set $\mathcal{C}$ for the case of disjoint clustering. Next, the typical $T$ rounds of communication akin to {\fedavg} will be run. The server sends across a model corresponding to each client in the (randomly) selected (time-varying) subset $\mathcal{S}^t$ (line 6). Similar to {\fedavg}, the client participation rate is defined as $\rho = |\mathcal{S}^t|/|\mathcal{S}|$. In case of disjoint clustering, all clients in cluster $\mathcal{C}_c$ will be given the same model $w_c^t$. Each client run $F$ epoch on its local data and sends back the updated \emph{local} model\, $\overbar{w}_m^{t+1}$. The top bar notation is used to denote client-side models. The server collects all the local models in \, $\overbar{W}^{t+1} = [\,\overbar{w}^{t+1}_1, \cdots,\, \overbar{w}^{t+1}_M]$ and updates them according to two possible cases. The first option (Case I) uses the full client relatedness matrix $\tilde{A}$ to benefit from the all the related client models (line 17). In this case, the sever updates the local model weights using $W^{t+1} =  \overbar{W}^t \, \tilde{A} \, \texttt{diag}(p_m / \|\tilde{A}_m\|_0)$. For notation simplicity, each model parameter set $w^{t+1}_m$ is assumed to be reshaped into a column vector. The second option (case II) is to define an update rule per disjoint cluster according to cluster membership in $\mathcal{C}$. In that case, standard {\fedavg} will be applied per cluster (line 19) and the aggregated model of the cluster, $w^{t + 1}_c$, will be sent to all the clients in the cluster in the next round (line 20). 

Last but not least, we reflect on an \emph{optional} dynamic clustering functionality of {\flt} in lines $13$, and $21$ to $28$ of Algorithm~\ref{alg:flt}. Here, we consider two possible circumstances: i) the data distribution of (some) clients varies suddenly or with time, e.g., new classes are introduced; ii) newcomer clients join the network. In such cases, the clients can \emph{optionally} send a flag of state change ($\delta := 1$) over along with their updated model parameters (line $13$) to notify the server of the change in their data or the state of being a newcomer. The server keeps track of these flags in $\Delta$ and once a certain number of clients have raised such flags (say $\lambda$ clients), it calls for repeating the cluster formation process. This needs further client cooperation and is prone to byzantine attacks. There are two other possibilities to do this without any client contribution. This can happen every $\tau$ rounds as denoted in line $18$ or can happen on the server side, by monitoring client model parameter change as determining factor, similar to the main clustering approach of \cite{xie2020multi, briggs2020federated}. In any case, if (e.g. $M'$) new clients are introduced and/or some clients have experienced data change, the server will update $\mathcal{S}$ accordingly (line $24$) and calls \texttt{FCR} with \texttt{MODE} set to fine-tune, possibly resulting in a new client relatedness graph and cluster membership ($\tilde{A}, \mathcal{C}$). The new clients will then be incorporated in the next communication round. Another possibility to handle newcomer clients is to assign them to the existing cluster formation $\mathcal{C}$ based on their distance to the (average) latent embeddings of the clusters. In favor of limited space, these directions are left as our future work.

\textbf{Convergence of {\flt}.} The following theorem generalizes convergence characteristics of {\fedavg} to that of {\flt} under same regularity assumptions. A detailed proof of this theorem is provided in Appendix~\ref{ssec:eval}.

\begin{theorem}
Let $G_\mathbb{W}(W^t) \overset{\Delta}{=} \| W^{t} - W^{t-1} \|^2$ represent the optimality gap for the stationary solution of {\flt}, where $W^t$ stands for model parameters in communication round $t$. Under common regularity assumptions \citep{wang2020}, the iterative gradient descent solution for {\flt} satisfies:
\begin{align} 
\label{eq:theo2}
\frac{1}{TM}&\sum_{i=1}^M \sum_{t = 1}^{T} \mathbb{E}[G_\mathbb{W}(W^t)] \nonumber \\ &\qquad \leq \frac{1}{T \times M}\frac{F(W^1) - F(W^{T})}{(\frac{1}{2\eta} - \frac{L_{W}}{2} - \frac{\eta}{2}L_{W}^2 \| I_M - \overbar{A} \|^2)},
\end{align}
where the cost function $F$ is assumed to be $L_W$-smooth, $I_M$ is the $M \times M$ identity matrix, and $\overbar{A}$ and $\eta$ denote the row-normalized adjacency matrix and the step size, respectively.
\end{theorem}

This theorem proves that akin to {\fedavg}, the convergence rate of {\flt} for a regular cost function is $\mathcal{O}(\frac{1}{T})$ for $T$ communication rounds. In Appendix~\ref{ssec:eval} besides the detailed proof of this theorem, we also provide a more elaborate convergence analysis for {\flt} where extra hyper-parameters such as the number of local epochs $E$ are incorporated.

%---------------------------------------------------
\section{Evaluation}
\label{sec:eval}
\vspace{-0.0cm}
%--------------------------------------------------
We now present empirical evidence on the impact of the proposed approach, {\flt}. We first describe our experimental setup covering a suite of synthetic and realistic (large-scale) non-IID scenarios. We then compare the performance (accuracy and fairness) of {\flt} against existing state-of-the-art baselines.

%---------------------------------------------------
\subsection{Experimental Setup}
\label{ssec:exp}
\vspace{-0.0cm}
%---------------------------------------------------
\textbf{Datasets.} We opt for image classification task as our downstream application of federated learning. For performance evaluation, we employ four different datasets: i) MNIST \cite{lecun1998mnist}, ii) CIFAR10 \cite{krizhevsky2014cifar}, iii) Federated EMNIST (FEMNIST) of LEAF \cite{caldas2018leaf} which is made out of Extended MNIST (EMNIST) \cite{cohen2017emnist}, iv) and our newly designed \emph{Structured Non-IID FEMNIST}. The latter is a resampled version of EMNIST to assess the performance in structured non-IID scenarios; see the supplementary materials for more detailed description of this new dataset. For experiments on MNIST and CIFAR10, we respectively use $60,000$ and $50,000$ samples for training, and $10,000$ for testing. MNIST and CIFAR10 contain $10$ classes of handwritten digits and objects with $6000$ and $5000$ samples per class, respectively. FEMNIST is a standard federated learning image classification dataset with $805,263$ samples that can accommodate up to $3550$ clients. Based upon EMNIST and similar to FEMNIST, we build a new dataset with (more pronounced) structured non-IID conditions and call it Structured Non-IID FEMNIST. To do so, we consider the ``balanced'' dataset of EMNIST, containing $131,600$ samples on $47$ classes. We use $112,800$ for training ($2400$ samples per class) and the remainder $18,800$ for testing. Depending on the scenarios explained later on, we partition the data into $C = 5$ or $10$ clusters. For experiments on FEMNIST, there is no predefined partitioning (or clustering) and we follow the standard definitions of LEAF \cite{caldas2018leaf}.

\textbf{Encoders.} To investigate the impact of encoder initialization, we also make use of three other datasets: a subset of CIFAR100 \cite{Krizhevsky09learningmultiple} (without any overlap with $10$ classes of CIFAR10), a $20$-class dataset composed of CIFAR10 and $10$ classes of CIFAR100 (we refer to this as ``CIFAR20''), and Fashion MNIST \cite{xiao2017fashion}. More concretely, we evaluate the performance of the proposed method in two encoder scenarios: \vspace{-0.0cm}
\begin{itemize}
\item[i)] \textbf{Enc1:} the encoder provided to the clients is pretrained on a large set of targets that covers the client target classes. This case is hoping for a holistic encoder on the server side. More specifically, for federated learning on MNIST we have used an encoder pretrained on EMNIST, and for CIFAR10 we use CIFAR20 introduced earlier.
\item[ii)] \textbf{Enc2:} the encoder is pretrained on a totally different dataset, and an initial fine-tuning per client would be required. This case demonstrates that lacking the holistic encoder (Enc1) is not a bottleneck for {\flt}. Note that in this scenario the encoder (ConvAE) is fine-tuned (also pre-trained) in an \emph{unsupervised} way and independent of the downstream federated learning task for $J$ more epochs based on a mean-squared-error (MSE) loss. In this case, for federated learning on MNIST we use an encoder pretrained on Fashion MNIST, and for CIFAR10 we use CIFAR100, which has no overlap with CIFAR10. For experiments on EMNIST/FEMNIST and Structured Non-IID FEMNIST, we only considered Enc2 pretrained on Fashion MNIST.  
\end{itemize}

\textbf{Network parameters.} We consider two model architectures for local client training, a multi-layer perceptron (MLP), and a convolutional neural network (CNN). We use an MLP with ReLU activation, and a single hidden layer of size $200$. We use the CNN of LEAF \cite{caldas2018leaf} with $2$ convolutional layers followed by $2$ fully connected layers. The encoder section of our frozen ConvAE has $2$ convolutional layers followed by a single fully-connected layer. Notice that the encoder is as simple as the local client model. See the supplementary materials for more details on model architectures. We set the number of local epochs to $E = 5$, and the total communication rounds to $T = 100$, unless otherwise mentioned. The local training is a mini-batch SGD with batch size of $10$ and learning rate $\eta = 0.01$. For {\flt}, the size of the latent embedding is $e = 128$ and $k$ in \texttt{kMEANS} is set to $5$. Even though on the client side $k$ can be adjusted according to the number of client classes. The number of fine-tuning epochs for Enc2 is set to $J = 5$, and $\gamma = 1$. The client participation fraction $\rho = 20\%$, unless otherwise mentioned. For hierarchical clustering, we use the Ward's method for linkage creation \cite{mullner2011modern}.

%---------------------------------------------------
\subsection{Evaluation Scenarios}
\label{ssec:eval_scenarios}
\vspace{-0.0cm}
%---------------------------------------------------
According to \cite{MAL-083, mcmahan2017communication}, there are several possible sources of non-IIDness in client data distributions. Among those, label distribution skew, or so-called ``pathological'' partitioning is the most commonly adopted approach in literature. In this case, different clients will have different class labels, which together with quantity skew (varying number of samples across clients) leads to most destructive impact on {\fedavg} \cite{hsu2019measuring}. We build the following scenarios upon these angles.

\textbf{Scenario 1 [MNIST, CIFAR10]:} We consider a network of $M = |\mathcal{S}| = 100$ clients clustered into $C = 5$ clusters. The training data samples will be evenly distributed among these clients, $600$ samples each ($500$ for CIFAR10), and the clients in each cluster will have samples only from two distinct classes. For instance, for MNIST, clients in cluster 1 ($\mathcal{C}_1$) will have only samples from digits ``\texttt{0}'' and ``\texttt{1}'', and those in $\mathcal{C}_2$ from  ``\texttt{4}'' and ``\texttt{7}'' without overlap with $\mathcal{C}_1$.

\textbf{Scenario2 [MNIST, CIFAR10]:} This scenario is the same as scenario 1, except that the clients in two different clusters can have 1 similar label/class. As an example, for MNIST, clients in cluster 1 ($\mathcal{C}_1$) will have only samples from digits ``\texttt{0}'', ``\texttt{1}'', and ``\texttt{2}'', and those in $\mathcal{C}_2$ can have samples from ``\texttt{2}'', ``\texttt{3}'' and ``\texttt{7}''. This is to investigate the performance in less extreme non-IID conditions.

\textbf{Scenario 3 [FEMNIST]:} We import the standard FEMNIST dataset of LEAF and construct a network of $200$ clients according to train and test data distributions defined in \cite{caldas2018leaf}. We run our experiments for a total of $T = 100$, $1000$ and $1500$ communication rounds. There are no predefined (or structured) clusters in FEMNIST and it is up to the federated learning method to form clusters, if need be. Here we consider both cases in Algorithm~\ref{alg:flt}, and assess the performance of {\flt} in the generic form (no clusters, case I) as well as with $C = 2$, $3$, $5$ and $7$ clusters (case II). 
	
\textbf{Scenario 4 [Structured Non-IID FEMNIST]:} As mentioned earlier, we introduced this dataset with the purpose of imposing (more extreme) structured non-IIDness in FEMNIST. To this aim, we impose label distribution skew (across clusters) as well as quantity skew following a power law for the number of samples per client in each cluster, akin to \cite{li2018federated}. We consider $C = 10$ clusters, each containing $5$ distinct character classes (total of $12,000$ data samples per cluster), except the last one containing $2$ classes ($4800$ samples), resulting in a total of $47$ classes and $112,800$ samples. We also consider even a larger network than Scenario 3 with $M = 2400$ clients ($240$ clients per cluster) for MLP, and $M = 600$ clients ($100$ clients per cluster) for CNN. Notably, as a result of our random sampling strategy, the clients in each cluster will have a random subset of the labels assigned to that cluster.  
	
%---------------------------------------------------
\subsection{Baselines and Competitors}
\label{ssec:baselines}
\vspace{-0.0cm}
%---------------------------------------------------
\textbf{Baselines.} We consider 5 baselines as described in the following. i) {\fedavg} \cite{mcmahan2017communication} where a \emph{single global model} is trained for the whole network. ii) \texttt{local} where each client trains its own model with its own local data. iii) \texttt{PCA+kM+HC} Inspired by \cite{Dennis2020hetero} (focused on federated clustering), this method applies linear principle component analysis (PCA) followed by \texttt{kMEANS} (at client-side) and hierarchical clustering (\texttt{HC}) at the server side. The goal is to illustrate the impact of the nonlinear encoder (convAE) and manifold learning (\texttt{UMAP}) components of {\flt} for discovering task-relatedness. iv, v) We also compare our performance with two of the most recent state-of-the-art clustered federated learning approaches called \texttt{IFCA} \cite{NEURIPS2020_e32cc80b} and \texttt{FedSEM} \cite{xie2020multi}. Notably, \texttt{FedSEM} already outperforms other recent baselines such as {\fedprox} \cite{li2018federated} and \texttt{CFL} \cite{sattler2020clustered}, and thus, has been selected as the outstanding approach. For the sake of reference though, we also report the result of \texttt{CFL} in Scenario 3. \texttt{FedSEM} constructs multiple clusters each building its own local cluster-level model. For each cluster, a virtual cluster center is defined and its parameters are updated in an iterative fashion. The key idea behind cluster formation is measuring the distance between clients and the virtual cluster center model parameters. \texttt{IFCA} tries to construct multiple clusters iteratively by alternating between cluster identification estimation and loss function minimization. It starts with initializing model parameters for cluster centers which are then broadcast to the randomly participating clients per communication round (a costly communication overhead). The clients estimate their cluster identity by finding the model that returns the lowest loss, and send their cluster identity along with the updated model. For both methods: i) the number of clusters has to be known apriori; ii) a good initialization of cluster center model parameters is key for convergence; iii) they require several communication rounds before clusters are formed. We will demonstrate in the next subsections that these characteristics can lead to slow convergence and sub-optimal performance in \emph{structured non-IID} scenarios.

\textbf{Fairness in federated learning.} Recently, \emph{fairness} in performance has become an important concern in federated learning. In this context, being ``fair'' is to avoid disproportionately advantaging or disadvantaging some of the clients. Among several interesting approaches to fairness, two recent ones stand out in the federated learning literature: i) best worst-case performance \cite{hashimoto2018fairness, pmlr-v97-mohri19a}, and ii) least variance across clients \cite{Li2020Fair}. We focus on the latter and report the variance of model accuracies across clients as a measure of fairness.

%---------------------------------------------------
\subsection{Evaluation Results for Scenarios 1 and 2}
\label{ssec:scenario_1n2}
\vspace{-0.0cm}
%---------------------------------------------------
% 
\begin{table}[t]
	\small
	\caption{Test accuracies ($\% \pm$ std. error) for Scenario 1.}
	\vspace{-0.2cm}
	\label{tb:scenario_1}
	\centering
	{\tabcolsep=0pt\def\arraystretch{1.0}
		\begin{tabularx}{240pt}{l *4{>{\Centering}X}}
			\toprule
			& \multicolumn{2}{c}{MNIST}   & \multicolumn{2}{c}{CIFAR10} 
			\tabularnewline \cmidrule(lr){2-3}\cmidrule(l){4-5}
			Method								& acc. 							&var. 	&acc.  	& var. \tabularnewline 
			\midrule
			{\fedavg}  					&82.73\scriptsize{$\pm$0.38} 	&44.65 	&33.29\scriptsize{$\pm$0.47} 		&108.20 \tabularnewline
			\texttt{Local}  					&97.41\scriptsize{$\pm$0.16}  	&3.44 		&79.81\scriptsize{$\pm$0.40} 	&115.35 \tabularnewline
			\texttt{PCA+kM+HC}  				&83.82\scriptsize{$\pm$0.37}  	&41.13 		&75.35\scriptsize{$\pm$0.40} 	&589.6 \tabularnewline
			{\fedsem}  					&98.01\scriptsize{$\pm$0.14} 	&2.12 	&78.23\scriptsize{$\pm$0.41} 		&695.01 \tabularnewline
			\rowcolor{LightCyan}
			{\flt} (Enc1)  				&97.98\scriptsize{$\pm$0.14} 	&2.17 	&87.11\scriptsize{$\pm$0.33} 		&88.80 \tabularnewline
			\rowcolor{LightCyan}
			{\flt} (Enc2) 				&97.97\scriptsize{$\pm$0.14} 	&2.17 	&87.17\scriptsize{$\pm$0.33}  		&84.62 \tabularnewline
			\bottomrule
	\end{tabularx}}\vspace{-0.0cm}
\vspace{-.0in}
\end{table}
\begin{table}[t]
	\small
	\caption{Test accuracies ($\% \pm$ std. error) for Scenario 2.}
    \vspace{-0.2cm}
	\label{tb:scenario_2}
	\centering
	{\tabcolsep=0pt\def\arraystretch{1.0}
		\begin{tabularx}{240pt}{l *4{>{\Centering}X}}
			\toprule
			& \multicolumn{2}{c}{MNIST}   & \multicolumn{2}{c}{CIFAR10} 
			\tabularnewline \cmidrule(lr){2-3}\cmidrule(l){4-5}
			Method 								& acc. & var. & acc.  & var. \tabularnewline 
			\midrule
			{\fedavg}  					&86.08\scriptsize{$\pm$0.35}    &17.55      & 52.77\scriptsize{$\pm$0.50}  &19.18 \tabularnewline
			\texttt{Local}  					&96.58\scriptsize{$\pm$0.17}    &1.81       &70.44\scriptsize{$\pm$0.46}   &72.29 \tabularnewline
			\texttt{PCA+kM+HC}  				&85.45\scriptsize{$\pm$0.35}    &18.42 		&61.75\scriptsize{$\pm$0.49}   &312.22 \tabularnewline
			{\fedsem}						&94.14\scriptsize{$\pm$0.23}    &1.20       &73.26\scriptsize{$\pm$0.44}   &339.39 \tabularnewline
			\rowcolor{LightCyan}
			{\flt} (Enc1)  				& 97.37\scriptsize{$\pm$0.16} &1.07 & 79.971\scriptsize{$\pm$0.40} &146.84 \tabularnewline
			\rowcolor{LightCyan}
			{\flt} (Enc2) 				& 97.37\scriptsize{$\pm$0.16} &1.07 & 80.00\scriptsize{$\pm$0.40} &78.66 \tabularnewline
			\bottomrule
	\end{tabularx}}
\vspace{-.0in}
\end{table}

Average model accuracies on \emph{test} data (and their standard errors) for Scenarios 1 and 2 after $T = 100$ communication rounds are shown in Tables~\ref{tb:scenario_1} and \ref{tb:scenario_2}. These are accompanied by the variance of model accuracies across clients which is serving as our fairness metric \cite{Li2020Fair}. For MNIST, the number of local epochs is set to $E = 1$ and for CIFAR10 is set to $E = 5$. Note that irrespective of the choice of encoder (Enc1 or Enc2), the proposed method ({\flt}) achieves roughly the same performance in terms of accuracy in both scenarios. This highlights that our method is reasonably robust against the choice of the encoder and the fine-tuning idea (starting from a totally different dataset and finetuning only for $J=5$ epochs) for Enc2 is functioning. As can be seen in both tables, even with $E = 1$, the differences between {\flt} and {\fedsem} are relatively smaller on MNIST but this discrepancy grows for CIFAR10. In summary, on MNIST and CIFAR10, we improve upon the state-of-the-art approach {\fedsem} by about $3\%$ and $9\%$ respectively, and way beyond that for the case of {\fedavg}, \texttt{Local} and \texttt{PCA+kM+HC}. On top of accuracy, {\flt} returns the least test accuracy variance (best fairness) for almost all settings in the two scenarios. 

%The \texttt{Local} models are overfitting leading to extremely high accuracy in training (illustrated on the convergence graphs in supplementary Section~\ref{sec:more_eval}) and not-so impressive in test results. 

There are two reasons behind achieving reasonable performance by \texttt{Local} models: i) in these two scenarios clients have a reasonably high number of samples ($600$ for MNIST, $500$ for CIFAR10); ii) the target classes in different clusters are almost independent. Notice that extreme non-IIDness (total class independence across clusters) in Scenario 1 helps \texttt{Local} models to score better than in Scenario 2 where some correlation between labels is allowed. A counter-argument applies to our method {\flt} because hard thresholding (mapping cluster membership to $0$'s and $1$'s) for $\Gamma$ in Algorithm~\ref{alg:formcluster} seems to be limiting the performance in handling inter-cluster dependencies. \texttt{Local} models do not illustrate competitive performance. Therefore, we omit them in the following experiments.   

%---------------------------------------------------
\subsection{Evaluation Results for Scenarios 3}
\label{ssec:scenario_3}
\vspace{-0.0cm}
%---------------------------------------------------
\begin{figure*}[t!]
	\centering
  	\hfill
	\begin{subfigure}[b]{0.32\textwidth}
		\centering
    	\includegraphics[trim={2.59cm 2.2cm 1.82cm 2.2cm},clip,width=\textwidth]{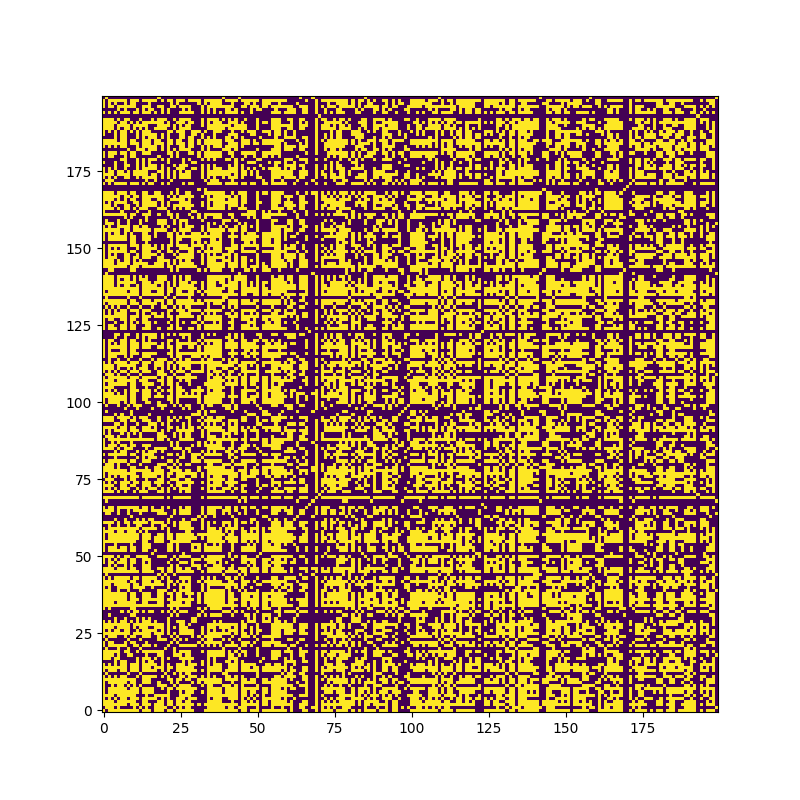}
    	\caption{Raw adjacency matrix.}
    \end{subfigure}
	\hfill
	\begin{subfigure}[b]{0.32\textwidth}
		\centering
	    \includegraphics[trim={1.62cm 1.9cm 1.62cm 1.87cm},clip,width=\textwidth]{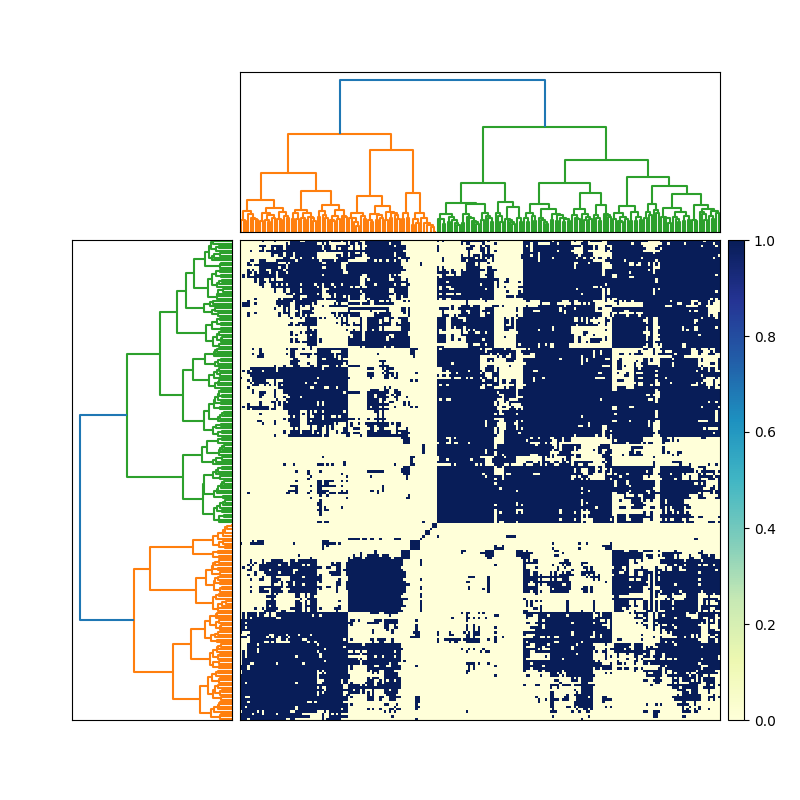}
    	\caption{Reordered with hierarchical clustering. }
    \end{subfigure}
    \hfill
    \begin{subfigure}[b]{0.32\textwidth}
		\centering
	    \includegraphics[trim={3.63cm 1.33cm 3.0cm 1.33cm},clip,width=\textwidth]{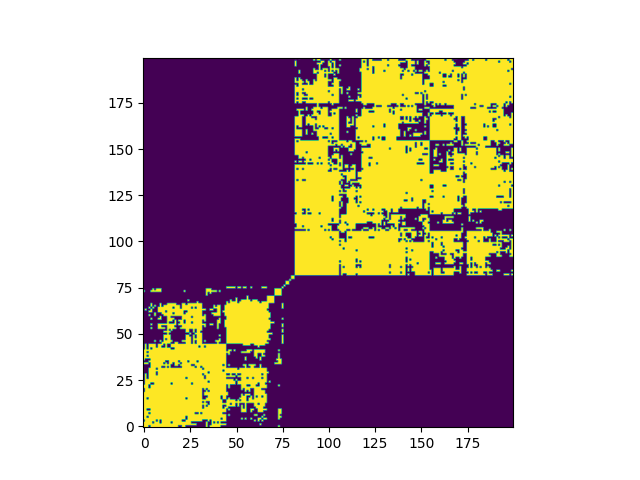}
    	\caption{Flattened to $C=2$ clusters.}
    \end{subfigure}
	\vspace{-0.2cm}
	\caption{Hierarchical clustering (\texttt{HC}) is used to reorder clients relatedness graph ($\tilde{A} \rightarrow \tilde{A}^r$) and extract near-optimal block-diagonal structures for disjoint cluster formation. In (b), the dendrogram highlights the client dependencies leading to $C=2$ disjoint clusters.}
	\label{fig:hierarchical}
	\vspace{-0.4cm}
\end{figure*}

Predefined data distribution of FEMNIST in LEAF \cite{caldas2018leaf} and its adoption by other state-of-the-art methods helped us to also benchmark our approach against {\ifca} \cite{NEURIPS2020_e32cc80b}. Here, we present performance results based on both MLP and CNN networks for $M = 200$ clients; a setting commonly adopted in literature. For {\flt}, we only present the more challenging Enc2 case where the encoder is pretrained on Fashion MNIST, and in a \emph{one-off process} each client fine-tunes the encoder (for only $J=5$ epochs). The authors of \cite{NEURIPS2020_e32cc80b} employ a slightly modified data sampling strategy; however, we use the standard data distribution of FEMNIST as in \cite{caldas2018leaf}. They also run for a total of $T = 6000$ communication rounds to reach the nominal performance, $2000$ of which is initialization with {\fedavg} for ``weight sharing'' for a smaller client participation rate ($\rho = 3\%$). For the sake of a fair comparison, we run $500$ rounds of initialization (with {\fedavg}) followed by another $1000$ rounds of {\ifca} itself, i.e. a total of $1500$ rounds. All the MLP experiments are run for $T = 1000$ communication rounds, and those for CNN are run for only $T= 100$ rounds, except for {\ifca} which is again run for $1500$ rounds for both MLP and CNN. Remember that FEMNIST does not introduce any predefined cluster structure, and thus, non-IIDness in the structured form. As such, this is up to the methodology to define clusters. Both {\ifca} and {\fedsem} must define clusters ($C > 1$) or with $C = 1$ they would degenerate to {\fedavg}; however, thanks to $\texttt{FCR}$ in Algorithm~\ref{alg:formcluster}, this is not the case for {\flt}. We adopt the setting leading to the best performance reported by both {\ifca} and {\fedsem} with $C=3$ and $C=5$ clusters, respectively. For {\flt} we present the results for both cases described in Section~\ref{sec:clustFL}: i) where the full client relatedness matrix will be used, and ii) where we use hierarchical clustering (\texttt{HC}) to specifically extract disjoint clusters (only if need be, as a special case). This is shown in Fig.~\ref{fig:hierarchical} where (a) shows the raw client relatedness matrix, (b) illustrates how \texttt{HC} would reorder this to extract cluster level dependencies (dendrogram), and (c) shows the flattened version with $C=2$ disjoint clusters.  

The \emph{test} accuracies and fairness measures are summarized in Table~\ref{tb:scenario_3}. By using the full client relatedness matrix, {\flt} beats all the competitors (by $+2\%$, $+12\%$, $+1.5\%$ for \texttt{PCA+kM+HC}, {\ifca}, {\fedsem}) for MLP and (by $+2\%$, $+0.5\%$, $+2\%$ for \texttt{PCA+kM+HC}, {\ifca}, {\fedsem}) in case of CNN local models. As also reported in \cite{xie2020multi}, {\fedsem} already beats {\cfl}, so does {\flt}. Notably, {\flt} significantly outperform {\ifca} in the MLP setting and marginally outperform in the case of CNN. However, in case of CNN, all models except {\ifca} are run for only $100$ communication rounds. In practice, it took {\ifca} $1500$ communication rounds to reach a comparable performance regime. For this reason, we omit {\ifca} in our next experiments. Note that \texttt{PCA+kM+HC} (inspired by the line of thought in \cite{Dennis2020hetero} and embedded in an end-to-end federated learning setting) also performs in par with the best methods and considerably better than {\ifca} in the MLP setting. For the sake of fair comparison, hierarchical clustering (\texttt{HC}) with $C=3$ clusters is applied here. The convergence graphs of average \emph{test} accuracies are illustrated in Figs.~\ref{fig:scenario_3} and ~\ref{fig:scenario_3_cnn} where {\flt} is the fastest in terms of convergence. We also investigated the impact of creating disjoint clusters with {\flt} (case II in Algorithm~\ref{alg:flt}) in this setting. Introducing clusters ($C=2$, $3$ and $5$) in this dataset seems to be slightly degrading the performance of {\flt}, even though it still remains to be in par with the best performing models. See the supplementary materials for more detailed results. Reflecting on the relatively smaller performance margin in this scenario, we argue that FEMNIST may not have a clear cluster \emph{structure}, and thus, cluster-based methods might not offer a significant gain. This also concurs with that {\flt} employing the full adjacency matrix slightly outperforms its clustered variants. This was the main motivation behind designing Scenario 4 to assess the potential of these algorithms in a more challenging large-scale setting: Structured Non-IID FEMNIST. 
\begin{table}[t!]
	\small
	\caption{Test accuracies ($\% \pm$ std. error) for Scenario 3.}
	\vspace{-0.2cm}
	\label{tb:scenario_3}
	\centering
	{\tabcolsep=0pt\def\arraystretch{1.0}
		\begin{tabularx}{240pt}{l *4{>{\Centering}X}}
			\toprule
			& \multicolumn{2}{c}{MLP}   & \multicolumn{2}{c}{CNN} 
			\tabularnewline \cmidrule(lr){2-3}\cmidrule(l){4-5}
			Method 								& acc. & var. & acc.  & var. \tabularnewline 
			\midrule
			{\fedavg}  				&  72.76\scriptsize{$\pm$0.76} & 202.61 & 81.64\scriptsize{$\pm$0.66}  &147.03 \tabularnewline
			{\texttt{PCA+kM+HC}\tiny{($C{=}3$)}}				  	&71.94\scriptsize{$\pm$0.77} & 194.75 &80.07\scriptsize{$\pm$0.68} &160.32 			\tabularnewline
			{\ifca} \cite{NEURIPS2020_e32cc80b}					&  61.24\scriptsize{$\pm$0.84} & 176.38 & 81.47\scriptsize{$\pm$0.66}  &118.71\tabularnewline
			{\cfl} \cite{sattler2020clustered} &65.35\scriptsize{$\pm$0.81} & 180.23 & 76.68\scriptsize{$\pm$0.73} & 132.42
			\tabularnewline
			{\fedsem} \cite{xie2020multi}				&  72.45\scriptsize{$\pm$0.76} & 185.96 & 79.99\scriptsize{$\pm$0.68}  &156.27\tabularnewline
			\rowcolor{LightCyan}
			{\flt}		 			&  74.11\scriptsize{$\pm$0.74} & 171.31 & 82.14\scriptsize{$\pm$0.65}  &145.03\tabularnewline
			\rowcolor{LightCyan}
			{\flt}\tiny{($C{=}2$)}  &   72.83\scriptsize{$\pm$0.76} & 175.55 & 80.53\scriptsize{$\pm$0.68}  &150.26\tabularnewline
			\rowcolor{LightCyan}
			{\flt}\tiny{($C{=}3$)}	&  72.61\scriptsize{$\pm$0.76} & 184.59 & 79.72\scriptsize{$\pm$0.69}  &170.10\tabularnewline
			\rowcolor{LightCyan}
			{\flt}\tiny{($C{=}5$)}	&  72.05\scriptsize{$\pm$0.77} & 180.32 & 78.71\scriptsize{$\pm$0.70}  &169.22\tabularnewline
% 			{\flt}\tiny{($C{=}7$)}	&  71.86\scriptsize{$\pm$0.26} & 187 & 78.70\scriptsize{$\pm$0.24}  &161.45\tabularnewline
			\bottomrule
	\end{tabularx}}
\vspace{-.0in}
\end{table}
\begin{table}[t!]
	\small
	\captionof{table}{Test accuracies ($\% \pm$ std. error) for Scenario 4.}
	\vspace{-0.2cm}
	\label{tb:scenario_4}
	\centering
	{\tabcolsep=0pt\def\arraystretch{1.0}
		\begin{tabularx}{240pt}{l *4{>{\Centering}X}}
%			& \multicolumn{2}{c}{EMNIST} 
			%\tabularnewline \cmidrule(lr){2-3}
			\toprule
			& \multicolumn{2}{c}{MLP}   & \multicolumn{2}{c}{CNN} 
			\tabularnewline \cmidrule(lr){2-3}\cmidrule(l){4-5}
			Method 								& acc. & var. & acc.  & var. \tabularnewline 
			\midrule
			{\fedavg}  				&  46.50\scriptsize{$\pm$0.36} & 100.27 &53.58\scriptsize{$\pm$0.36} &883.47
			\tabularnewline
			%\texttt{Local}  				&  71.67\scriptsize{$\pm$0.33} & 131.97 \tabularnewline
			{\texttt{PCA+kM+HC}}				  	&48.75\scriptsize{$\pm$0.36} &467.92  &59.87\scriptsize{$\pm$0.36} &912.39			\tabularnewline
			{\fedsem}  \cite{xie2020multi}				  	&  43.53\scriptsize{$\pm$0.36} & 406.60 &76.22\scriptsize{$\pm$0.29} &1264.38
			\tabularnewline
			\rowcolor{LightCyan}
			{\flt} 		 			&  86.51\scriptsize{$\pm$0.24} & 207.60 &93.69\scriptsize{$\pm$0.10} &98.65 \tabularnewline
			\bottomrule
    	\end{tabularx}}
\vspace{-.0in}
\end{table}
\begin{figure*}[t!]
\centering
    \begin{subfigure}{.49\linewidth}
    	\centering
    	\includegraphics[width=0.99\textwidth]{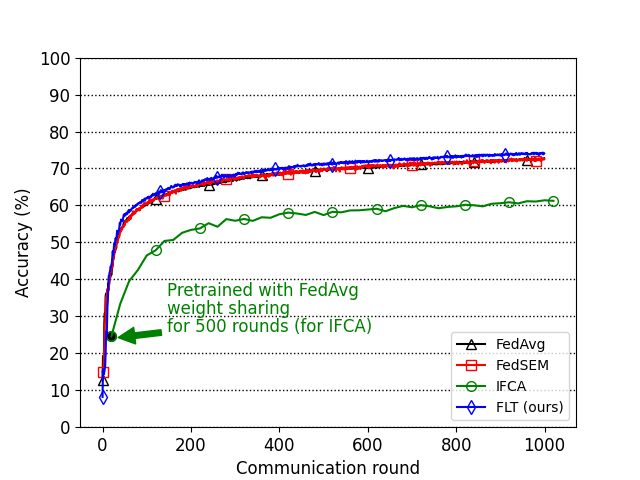}
    	\vspace{-0.2cm}
    	\caption{Test accuracies for MLP.}
    	\label{fig:scenario_3}
    \end{subfigure}
\hfill
    \begin{subfigure}{.49\linewidth}
    	\centering
    	\includegraphics[width=0.99\textwidth]{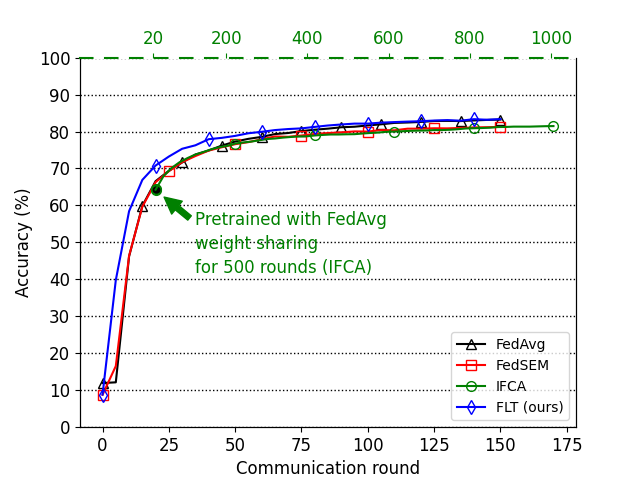}
    	\vspace{-0.2cm}
    	\caption{Test accuracies for CNN. }
    	\label{fig:scenario_3_cnn}
    \end{subfigure}
    \vspace{-0.2cm}
    \caption{Convergence graph of \emph{test} accuracies for Scenario 3, FEMNIST, $M=200$ (left: MLP, right, CNN). On the right for CNN, note the different range of communication rounds on the top horizontal axis associated with {\ifca}.}
\end{figure*}

%---------------------------------------------------
\subsection{Evaluation Results for Scenarios 4}
\label{ssec:scenario_4}
\vspace{-0.0cm}
%---------------------------------------------------
As explained earlier, Scenario 4 presents a large-scale federated learning setting with structured non-IIDness involving both quantity and label distribution skews. For {\flt}, we only present the more challenging Enc2 case where the encoder is pretrained on Fashion MNIST, and in a \emph{one-off process} each client fine-tunes the encoder (for only $J = 5$ epochs). Convergence graphs of average \emph{test} accuracies are shown in Fig.~\ref{fig:scenario_4} and \ref{fig:scenario_4_cnn}, for MLP and CNN networks, respectively. The \emph{test} accuracies (and their standard error) together with the fairness measure (variance across clients) for at the last communication round $T = 100$ are  summarized in Table~\ref{tb:scenario_4}. Interestingly, the state-of-the-art competitor, {\fedsem}, suffers in the MLP scenario and roughly scores as good as the barebone {\fedavg}, which is not made to cope with structured non-IID scenarios. The performance of {\fedsem} is noticeably improved in the CNN setting, but accordingly {\flt} is boosted as well, reaching $+93\%$ \emph{test} accuracy in this challenging scenario. Here, we beat {\fedsem} by $+40\%$ and $+17\%$ in MLP and CNN scenarios, respectively. The main reason behind this downgrade in performance of {\fedsem} (compared to Scenario 3) is the struggle to discover and form clusters by relying on model parameter comparisons. This is in turn due to two main factors: i) significantly larger number of models (per cluster and across) for this extreme non-IID scenario leads to tremendous heterogeneity in model space and thus considerable increase in complexity of pairwise model comparisons; ii) the problem is exacerbated due to the limited number of samples provided to some clients (down to $20$ samples because of the power law) resulting in lower-quality local model training. An evidence confirming this hypothesis is that these problems are even more pronounced in the case of a simpler local model (MLP), where {\fedsem} falls almost back to {\fedavg}. On the other hand, thanks to $\texttt{FCR}$ on client side, and \texttt{UMAP} and \texttt{HC}) on server side, {\flt} manages to automatically detect $C=10$ clusters (See Appendix~\ref{sec:conv_issues} for more results). Both {\fedavg} and {\fedsem} can benefit from more communication rounds in this scenario. Nonetheless, the main message of this experiment is clear: the one-shot taskonomy-based client relatedness helps {\flt} to converge already in roughly $20$ communication rounds and outperform the other baselines by a significant margin, while also offering a lower variance across clients (best fairness) compared to most competitors.

\begin{figure*}[t!]
\centering
    \begin{subfigure}{.49\linewidth}
    	\centering
    	\includegraphics[width=0.9\textwidth]{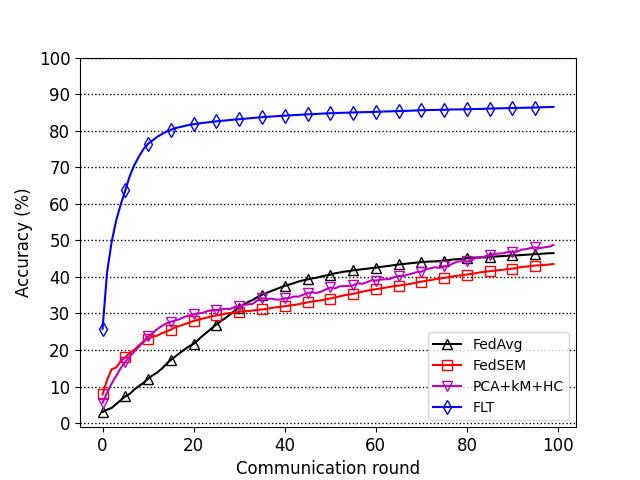}
    	\vspace{-0.2cm}
    	\caption{Test accuracies for MLP, $M = 2400$.}
    	\label{fig:scenario_4}
    \end{subfigure}
\hfill
    \begin{subfigure}{.49\linewidth}
    	\centering
    	\includegraphics[width=0.9\textwidth]{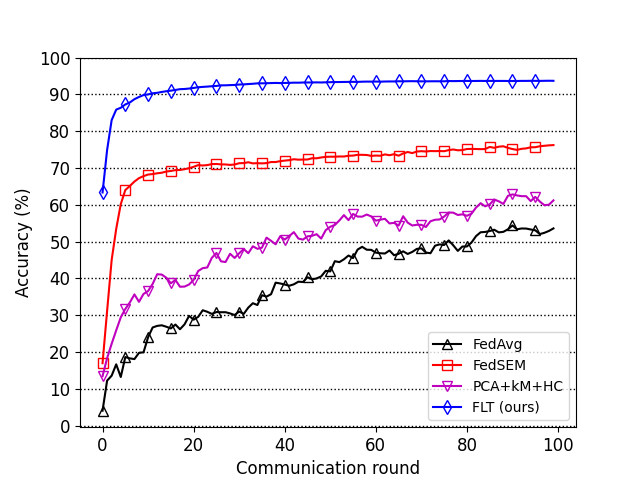}
    	\vspace{-0.2cm}
    	\caption{Test accuracies for CNN, $M = 600$.}
    	\label{fig:scenario_4_cnn}
    \end{subfigure}
    \vspace{-0.2cm}
    \caption{Convergence graph of \emph{test} accuracies for Scenario 4, Structured Non-IID FEMNIST, $C=10$.}
\end{figure*}

%---------------------------------------------------
\subsection{Impact of Hyper-parameter Changes}
\label{ssec:ablation}
\vspace{-0.0cm}
%---------------------------------------------------
We investigate the impact of a few remaining important hyper-parameters on the performance of {\fedsem} and {\flt} with both Enc1 and Enc2 initialization. We ran this experiment on Scenario 2 due to being an intermediate case where label overlap is allowed. Following this scenario, we set $C = 5$, and focus on CIFAR10, the more challenging dataset. The results are summarized in Table~\ref{tb:hyper}. Here, we consider $M = 100$ and $500$, client participation fraction per communication round $\rho = 20\%$ and $50\%$ (denoted as $0.2$ and $0.5$), $E = 5$, $J = 5$ (number of fine-tuning steps in Enc2 initialization), and we report the \emph{test} accuracy after $T = 50$ and $100$ communication rounds. As can be seen in the table, changing the total number of clients $M$ does not change the performance significantly (by looking at the top and bottom halves of the table), whereas increasing the communication rounds $T$ leads to better performance. Besides, for the same $M$, increasing the client participation fraction $\rho$ improves the overall performance of both methods. Furthermore, for the same $M$ and $\rho$, increasing the number of local epochs $E$ results in a superior performance, all of which is in line with the more elaborate version of Theorem\,$2$, in Appendix~\ref{ssec:eval}. Without exception, our method outperforms {\fedsem} by roughly $2\%$ to $8\%$.

%
% \begin{table}[t!]
% 	\small
% 	\captionof{table}{Ablation study of \emph{test} acc. ($\%$) for Scenario 2}
% 	\vspace{-0.1cm}
% 	\label{tb:hyper}
% 	\centering
% 	{\tabcolsep=0pt\def\arraystretch{1.0}
% 		\begin{tabularx}{240pt}{l *4{>{\Centering}X}}\toprule
% 		($M$, $\rho$, $E$, $T$) &  {\fedsem} \cite{xie2020multi}	& {\flt}  \tabularnewline 
% 		\midrule
% 		(100, 0.2, 5, 50)			&71.85\scriptsize{$\pm$0.45}  		&79.45\scriptsize{$\pm$0.40}\tabularnewline
% 		(100, 0.2, 5, 100)			&73.42\scriptsize{$\pm$0.44}  		&79.95\scriptsize{$\pm$0.40}\tabularnewline
% 		(100, 0.5, 5, 50)			&74.29\scriptsize{$\pm$0.44}		&80.02\scriptsize{$\pm$0.40}\tabularnewline
% 		(100, 0.5, 5, 100)			&74.15\scriptsize{$\pm$0.44}		&79.72\scriptsize{$\pm$0.40}\tabularnewline
% 		\midrule
% 		(500, 0.2, 5, 50)			&71.24\scriptsize{$\pm$0.44} 		&80.27\scriptsize{$\pm$0.40}\tabularnewline
% 		(500, 0.2, 5, 100)			&73.21\scriptsize{$\pm$0.45}		&81.77\scriptsize{$\pm$0.39}\tabularnewline
% 		(500, 0.5, 5, 50)			&76.79\scriptsize{$\pm$0.42}		&81.97\scriptsize{$\pm$0.37}\tabularnewline
% 		(500, 0.5, 5, 100)			&76.75\scriptsize{$\pm$0.42}		&81.57\scriptsize{$\pm$0.39}\tabularnewline
% 		\bottomrule
% 		\end{tabularx}}
% \vspace{-.0in}
% \end{table}

\begin{table}[t!]
    \small
	\captionof{table}{Ablation study of \emph{test} acc. ($\% \pm$ std.) for Scn. 2}
	%\vspace{-0.1cm}
	\label{tb:hyper}
	\centering
	{\tabcolsep=0pt\def\arraystretch{1.0}
		\begin{tabularx}{240pt}{l *4{>{\Centering}X}}\toprule
		& \multirow{2}{*}{\fedsem\,\cite{xie2020multi}}   & \multicolumn{2}{c}{\fltours}
			\tabularnewline \cmidrule(l){3-4}
			($M$, $\rho$, $E$, $T$) 								&  & Enc1  & Enc2 \scriptsize{($J=5$)}
		\tabularnewline 
		\midrule
		(100, 0.2, 1, 50)			&64.32\scriptsize{$\pm$0.48}  		  &66.23\scriptsize{$\pm$0.47}  &65.93\scriptsize{$\pm$0.47} \tabularnewline
		(100, 0.2, 5, 50)			&71.85\scriptsize{$\pm$0.45}  		&79.45\scriptsize{$\pm$0.40}  &78.64\scriptsize{$\pm$0.41}\tabularnewline
		(100, 0.2, 1, 100)			&69.33\scriptsize{$\pm$0.46}  		&72.44\scriptsize{$\pm$0.45}  &72.62\scriptsize{$\pm$0.45}\tabularnewline
		(100, 0.2, 5, 100)			&73.42\scriptsize{$\pm$0.44}  		&79.95\scriptsize{$\pm$0.40}  &79.62\scriptsize{$\pm$0.40}\tabularnewline
		(100, 0.5, 1, 50)			&72.94\scriptsize{$\pm$0.44}		&73.78\scriptsize{$\pm$0.44}  &73.17\scriptsize{$\pm$0.44}\tabularnewline
		(100, 0.5, 5, 50)			&74.29\scriptsize{$\pm$0.44}		&80.02\scriptsize{$\pm$0.40}  &79.22\scriptsize{$\pm$0.41}\tabularnewline
		(100, 0.5, 1, 100)			&77.37\scriptsize{$\pm$0.42}		&79.30\scriptsize{$\pm$0.41}  &78.23\scriptsize{$\pm$0.41}\tabularnewline
		(100, 0.5, 5, 100)			&74.15\scriptsize{$\pm$0.44}		&79.72\scriptsize{$\pm$0.40}  &79.45\scriptsize{$\pm$0.40}\tabularnewline
		\midrule
		(500, 0.2, 1, 50)			&64.88\scriptsize{$\pm$0.48} 		&66.69\scriptsize{$\pm$0.47}  &66.74\scriptsize{$\pm$0.47}\tabularnewline
		(500, 0.2, 5, 50)			&71.24\scriptsize{$\pm$0.44} 		&80.27\scriptsize{$\pm$0.40}  &79.77\scriptsize{$\pm$0.40}\tabularnewline
		(500, 0.2, 1, 100)			&70.65\scriptsize{$\pm$0.46}		&72.55\scriptsize{$\pm$0.46}  &72.84\scriptsize{$\pm$0.44}\tabularnewline
		(500, 0.2, 5, 100)			&73.21\scriptsize{$\pm$0.45}		&81.77\scriptsize{$\pm$0.39}  &81.00\scriptsize{$\pm$0.39}\tabularnewline
		(500, 0.5, 1, 50)			&72.26\scriptsize{$\pm$0.45}		&73.80\scriptsize{$\pm$0.44}  &73.88\scriptsize{$\pm$0.44}\tabularnewline
		(500, 0.5, 5, 50)			&76.79\scriptsize{$\pm$0.42}		&81.97\scriptsize{$\pm$0.37}  &81.08\scriptsize{$\pm$0.39}\tabularnewline
		(500, 0.5, 1, 100)			&76.82\scriptsize{$\pm$0.42}		&79.66\scriptsize{$\pm$0.40}  &79.01\scriptsize{$\pm$0.41}\tabularnewline
		(500, 0.5, 5, 100)			&76.75\scriptsize{$\pm$0.42}		&81.57\scriptsize{$\pm$0.39}  &81.22\scriptsize{$\pm$0.39}\tabularnewline
		\bottomrule
		\end{tabularx}}
\vspace{-.0in}
\end{table}

%
%---------------------------------------------------
\section{Concluding Remarks}
\label{sec:disc}
\vspace{-0.0cm}
%---------------------------------------------------
\textbf{Summary and future directions.} We proposed {\flt} that comes with the following notable advantages. First, it is one-shot and considerably faster in convergence compared to its competitors, especially in structured non-IID scenarios. Second, the problem formulation of {\flt} can handle both generic client relatedness (no specific clustering) as well as disjoint clusters, if need be. Third, in contrast to most existing baselines, it does not require prior knowledge about number of clusters to form them. Fourth, it performs slightly better than the state-of-the-art baselines in standard federated learning settings and significantly outperforms them in structured non-IID scenarios. Finally, {\flt} offers improved fairness (least performance disparity among clients) compared to the existing baselines in most presented scenarios. Last but not least, in Appendix~\ref{sec:math_support}, we also provide a detailed convergence proof for {\flt} under common assumptions required for the convergence of {\fedavg}. We are currently working on extending {\flt} to handle dynamic clustering in the presence of newcomers or time-varying data distributions, as is briefly discussed in Algorithm~\ref{alg:flt}. Another avenue to explore is removing the hard threshold $\Gamma$ and modifying hierarchical clustering (HC) to work with soft (non-binary) adjacency matrices.
\begin{table*}[t!]
	\footnotesize
	\captionof{table}{Communication complexity analysis.}
	\vspace{-0.1cm}
	\label{tb:complexity}
	\centering
	{\tabcolsep=0pt\def\arraystretch{1.0}
		\begin{tabularx}{310pt}{l *4{>{\Centering}X}}\toprule
		{\flt} &  {\fedsem} \cite{xie2020multi}	& {\ifca} \cite{NEURIPS2020_e32cc80b}  \tabularnewline 
		\midrule
		$\underbrace{M*W_{\textup{enc}} + k\,M\,e}_\textup{one-off} + 2\rho M W_{\textup{local}} T $	& $2\rho M W_{\textup{local}} T$		& $\rho M W_{\textup{local}} T (C+ 1)$ \tabularnewline
		\bottomrule
		\end{tabularx}}
\vspace{-.0in}
\end{table*}

\textbf{Complexity and practical considerations.} {\flt} introduces a \emph{one-off} overhead due to the client relatedness discovery process ($\texttt{FCR}$, Algorithm~\ref{alg:formcluster}). However, owing to \texttt{FCR}, it is faster than the existing iterative baselines and less prone to convergence issues, as we have demonstrated throughout Section~\ref{sec:eval} and Appendix~\ref{sec:conv_issues}. One can argue that this step can be prone to security issues during the uplink communication (akin to standard {\fedavg} communications). A possible solution to address this is adding encryption and client ID verification processes, which are outside the scope of our work \cite{geyer2017differentially, bonawitz2017practical, aono2017privacy}. From communication complexity perspective, this overhead requires the server to send an encoder model ($W_\textup{enc}$) to the clients, and the clients to send an array of size $k e$ (with $k$ in $k$-means and $e$ denoting the latent embedding dimension of the encoder) to the server. A rough estimate of the communication complexity of the proposed {\flt}, and two discussed state-of-the-art competitors ({\fedsem} and {\ifca}) is summarized Table~\ref{tb:complexity}. As can be seen, the communication complexity of {\flt} and {\fedsem} are essentially the same except for the first two one-off terms (without $T$ for total communication rounds) which could be neglected. Note that this is an initialization step and it can also happen in multiple steps. Excluding a few clients from this process due to for instance their unavailability, does not impact the performance of $\texttt{FCR}$ and in turn {\flt}. On the other hand, {\ifca} mandates roughly $(C + 1)/2$ times ($C$ being the number of clusters) more communication complexity. This is because in every communication round, $C$ virtual center models will have to be sent to all the participating clients. From compute complexity perspective, possible \emph{fine-tuning} of the encoder is only for a small number of epochs ($J = 5$ in our experiments) on an encoder which is as simple as the local client models; and this is yet another one-off process that can be neglected over long runs. When the number of clients grows to tens of thousands (in very large-scale networks), {\flt} has the flexibility to decompose the client relatedness graph into disjoint clusters and degenerate to the same complexity level its competitors inflict.

%---------------------------------------------------
\section{Acknowledgment}
\label{sec:ack}
\vspace{-0.0cm}
%---------------------------------------------------
The authors would like to thank Shell Global Solutions International B.V. and Delft University of Technology (TU Delft) for their support and for the permission to publish this work. The authors also thank Ahmad Beirami from Facebook AI for helpful discussions.

\typeout{}
\bibliography{FL}

\begin{thebibliography}{10}

\bibitem{konevcny2015federated}
J.~Kone{\v{c}}n{\`y}, B.~McMahan, and D.~Ramage, ``Federated optimization:
  Distributed optimization beyond the datacenter,'' {\em arXiv preprint
  arXiv:1511.03575}, 2015.

\bibitem{konevcny2016federated}
J.~Kone{\v{c}}n{\`y}, H.~B. McMahan, F.~X. Yu, P.~Richt{\'a}rik, A.~T. Suresh,
  and D.~Bacon, ``Federated learning: Strategies for improving communication
  efficiency,'' {\em arXiv preprint arXiv:1610.05492}, 2016.

\bibitem{MAL-083}
P.~Kairouz and H.~B.~M. et~al., ``Advances and open problems in federated
  learning,'' {\em Foundations and Trends in Machine Learning}, vol.~14,
  no.~1–2, pp.~1--210, 2021.

\bibitem{mcmahan2017communication}
B.~McMahan, E.~Moore, D.~Ramage, S.~Hampson, and B.~A. y~Arcas,
  ``Communication-efficient learning of deep networks from decentralized
  data,'' in {\em Artificial Intelligence and Statistics}, pp.~1273--1282,
  PMLR, 2017.

\bibitem{sattler2019robust}
F.~Sattler, S.~Wiedemann, K.-R. M{\"u}ller, and W.~Samek, ``Robust and
  communication-efficient federated learning from non-iid data,'' {\em IEEE
  transactions on neural networks and learning systems}, 2019.

\bibitem{geyer2017differentially}
R.~C. Geyer, T.~Klein, and M.~Nabi, ``Differentially private federated
  learning: A client level perspective,'' {\em arXiv preprint
  arXiv:1712.07557}, 2017.

\bibitem{bagdasaryan2020backdoor}
E.~Bagdasaryan, A.~Veit, Y.~Hua, D.~Estrin, and V.~Shmatikov, ``How to backdoor
  federated learning,'' in {\em International Conference on Artificial
  Intelligence and Statistics}, pp.~2938--2948, PMLR, 2020.

\bibitem{pmlr-v97-mohri19a}
M.~Mohri, G.~Sivek, and A.~T. Suresh, ``Agnostic federated learning,'' in {\em
  Proceedings of the 36th International Conference on Machine Learning}
  (K.~Chaudhuri and R.~Salakhutdinov, eds.), vol.~97 of {\em Proceedings of
  Machine Learning Research}, pp.~4615--4625, PMLR, 09--15 Jun 2019.

\bibitem{Li2020Fair}
T.~Li, M.~Sanjabi, A.~Beirami, and V.~Smith, ``Fair resource allocation in
  federated learning,'' in {\em International Conference on Learning
  Representations}, 2020.

\bibitem{48698}
A.~T. Suresh, B.~McMahan, P.~Kairouz, and Z.~Sun, ``Can you really backdoor
  federated learning?,'' 2019.

\bibitem{wang2020attack}
H.~Wang, K.~Sreenivasan, S.~Rajput, H.~Vishwakarma, S.~Agarwal, J.-y. Sohn,
  K.~Lee, and D.~Papailiopoulos, ``Attack of the tails: Yes, you really can
  backdoor federated learning,'' {\em Advances in Neural Information Processing
  Systems}, vol.~33, 2020.

\bibitem{li2020federated}
T.~Li, A.~K. Sahu, A.~Talwalkar, and V.~Smith, ``Federated learning:
  Challenges, methods, and future directions,'' {\em IEEE Signal Processing
  Magazine}, vol.~37, no.~3, pp.~50--60, 2020.

\bibitem{li2018federated}
T.~Li, A.~K. Sahu, M.~Zaheer, M.~Sanjabi, A.~Talwalkar, and V.~Smith,
  ``Federated optimization in heterogeneous networks,'' {\em arXiv preprint
  arXiv:1812.06127}, 2018.

\bibitem{smith2017federated}
V.~Smith, C.-K. Chiang, M.~Sanjabi, and A.~S. Talwalkar, ``Federated multi-task
  learning,'' {\em Advances in neural information processing systems}, vol.~30,
  pp.~4424--4434, 2017.

\bibitem{jiang2019improving}
Y.~Jiang, J.~Kone{\v{c}}n{\`y}, K.~Rush, and S.~Kannan, ``Improving federated
  learning personalization via model agnostic meta learning,'' {\em arXiv
  preprint arXiv:1909.12488}, 2019.

\bibitem{NEURIPS2020_24389bfe}
A.~Fallah, A.~Mokhtari, and A.~Ozdaglar, ``Personalized federated learning with
  theoretical guarantees: A model-agnostic meta-learning approach,'' in {\em
  Advances in Neural Information Processing Systems} (H.~Larochelle,
  M.~Ranzato, R.~Hadsell, M.~F. Balcan, and H.~Lin, eds.), vol.~33,
  pp.~3557--3568, Curran Associates, Inc., 2020.

\bibitem{li2021model}
Q.~Li, B.~He, and D.~Song, ``Model-contrastive federated learning,'' in {\em
  Proceedings of the IEEE/CVF Conference on Computer Vision and Pattern
  Recognition}, pp.~10713--10722, 2021.

\bibitem{NEURIPS2020_e32cc80b}
A.~Ghosh, J.~Chung, D.~Yin, and K.~Ramchandran, ``An efficient framework for
  clustered federated learning,'' in {\em Advances in Neural Information
  Processing Systems}, vol.~33, pp.~19586--19597, Curran Associates, Inc.,
  2020.

\bibitem{mansour2020three}
Y.~Mansour, M.~Mohri, J.~Ro, and A.~T. Suresh, ``Three approaches for
  personalization with applications to federated learning,'' {\em arXiv
  preprint arXiv:2002.10619}, 2020.

\bibitem{sattler2020clustered}
F.~Sattler, K.-R. M{\"u}ller, and W.~Samek, ``Clustered federated learning:
  model-agnostic distributed multitask optimization under privacy
  constraints,'' {\em IEEE Transactions on Neural Networks and Learning
  Systems}, 2020.

\bibitem{briggs2020federated}
C.~{Briggs}, Z.~{Fan}, and P.~{Andras}, ``Federated learning with hierarchical
  clustering of local updates to improve training on non-iid data,'' in {\em
  2020 International Joint Conference on Neural Networks (IJCNN)}, pp.~1--9,
  2020.

\bibitem{xie2020multi}
M.~Xie, G.~Long, T.~Shen, T.~Zhou, X.~Wang, and J.~Jiang, ``Multi-center
  federated learning,'' {\em arXiv preprint arXiv:2005.01026}, 2020.

\bibitem{zamir2018taskonomy}
A.~R. Zamir, A.~Sax, W.~Shen, L.~J. Guibas, J.~Malik, and S.~Savarese,
  ``Taskonomy: Disentangling task transfer learning,'' in {\em Proceedings of
  the IEEE conference on computer vision and pattern recognition},
  pp.~3712--3722, 2018.

\bibitem{zhao2018federated}
Y.~Zhao, M.~Li, L.~Lai, N.~Suda, D.~Civin, and V.~Chandra, ``Federated learning
  with non-iid data,'' {\em arXiv preprint arXiv:1806.00582}, 2018.

\bibitem{karimireddy2020scaffold}
S.~P. Karimireddy, S.~Kale, M.~Mohri, S.~Reddi, S.~Stich, and A.~T. Suresh,
  ``Scaffold: Stochastic controlled averaging for federated learning,'' in {\em
  International Conference on Machine Learning}, pp.~5132--5143, PMLR, 2020.

\bibitem{yeganeh2020inverse}
Y.~Yeganeh, A.~Farshad, N.~Navab, and S.~Albarqouni, ``Inverse distance
  aggregation for federated learning with non-iid data,'' in {\em Domain
  Adaptation and Representation Transfer, and Distributed and Collaborative
  Learning}, pp.~150--159, Springer, 2020.

\bibitem{laguel2020device}
Y.~Laguel, K.~Pillutla, J.~Malick, and Z.~Harchaoui, ``Device heterogeneity in
  federated learning: A superquantile approach,'' {\em arXiv preprint
  arXiv:2002.11223}, 2020.

\bibitem{zhao2021federated}
Z.~Zhao, C.~Feng, W.~Hong, J.~Jiang, C.~Jia, T.~Q. Quek, and M.~Peng,
  ``Federated learning with non-iid data in wireless networks,'' {\em IEEE
  Transactions on Wireless Communications}, 2021.

\bibitem{duan2020fedgroup}
M.~Duan, D.~Liu, X.~Ji, R.~Liu, L.~Liang, X.~Chen, and Y.~Tan, ``Fedgroup:
  Efficient clustered federated learning via decomposed data-driven measure,''
  2021.

\bibitem{Dennis2020hetero}
D.~K. Dennis, T.~Li, and V.~Smith, ``Heterogeneity for the win: One-shot
  federated clustering,'' in {\em Proceedings of the 38th International
  Conference on Machine Learning}, vol.~139 of {\em Proceedings of Machine
  Learning Research}, pp.~2611--2620, PMLR, 18--24 Jul 2021.

\bibitem{lloyd1982least}
S.~Lloyd, ``Least squares quantization in pcm,'' {\em IEEE transactions on
  information theory}, vol.~28, no.~2, pp.~129--137, 1982.

\bibitem{mcinnes2018umap}
L.~McInnes, J.~Healy, and J.~Melville, ``Umap: Uniform manifold approximation
  and projection for dimension reduction,'' {\em arXiv preprint
  arXiv:1802.03426}, 2018.

\bibitem{stich2018local}
S.~U. Stich, ``Local {SGD} converges fast and communicates little,'' in {\em
  International Conference on Learning Representations}, 2019.

\bibitem{wang2018cooperative}
J.~Wang and G.~Joshi, ``Cooperative sgd: A unified framework for the design and
  analysis of communication-efficient sgd algorithms,'' {\em arXiv preprint
  arXiv:1808.07576}, 2018.

\bibitem{woodworth2018graph}
B.~E. Woodworth, J.~Wang, A.~Smith, B.~McMahan, and N.~Srebro, ``Graph oracle
  models, lower bounds, and gaps for parallel stochastic optimization,'' {\em
  Advances in neural information processing systems}, vol.~31, pp.~8496--8506,
  2018.

\bibitem{mullner2011modern}
D.~M{\"u}llner, ``Modern hierarchical, agglomerative clustering algorithms,''
  {\em arXiv preprint arXiv:1109.2378}, 2011.

\bibitem{wang2020}
S.~Wang and T.-H. Chang, ``Federated clustering via matrix factorization
  models: From model averaging to gradient sharing,'' {\em arXiv preprint
  arXiv:2002.04930}, 2020.

\bibitem{lecun1998mnist}
Y.~LeCun, ``The mnist database of handwritten digits,'' {\em http://yann.
  lecun. com/exdb/mnist/}, 1998.

\bibitem{krizhevsky2014cifar}
A.~Krizhevsky, V.~Nair, and G.~Hinton, ``The cifar-10 dataset,'' {\em online:
  http://www. cs. toronto. edu/kriz/cifar. html}, vol.~55, 2014.

\bibitem{caldas2018leaf}
S.~Caldas, S.~M.~K. Duddu, P.~Wu, T.~Li, J.~Konečný, H.~B. McMahan, V.~Smith,
  and A.~Talwalkar, ``Leaf: A benchmark for federated settings,'' 2019.

\bibitem{cohen2017emnist}
G.~Cohen, S.~Afshar, J.~Tapson, and A.~Van~Schaik, ``Emnist: Extending mnist to
  handwritten letters,'' in {\em 2017 International Joint Conference on Neural
  Networks (IJCNN)}, pp.~2921--2926, IEEE, 2017.

\bibitem{Krizhevsky09learningmultiple}
A.~Krizhevsky, G.~Hinton, {\em et~al.}, ``Learning multiple layers of features
  from tiny images,'' {\em Citeseer}, 2009.

\bibitem{xiao2017fashion}
H.~Xiao, K.~Rasul, and R.~Vollgraf, ``Fashion-mnist: a novel image dataset for
  benchmarking machine learning algorithms,'' {\em arXiv preprint
  arXiv:1708.07747}, 2017.

\bibitem{hsu2019measuring}
T.-M.~H. Hsu, H.~Qi, and M.~Brown, ``Measuring the effects of non-identical
  data distribution for federated visual classification,'' {\em arXiv preprint
  arXiv:1909.06335}, 2019.

\bibitem{hashimoto2018fairness}
T.~Hashimoto, M.~Srivastava, H.~Namkoong, and P.~Liang, ``Fairness without
  demographics in repeated loss minimization,'' in {\em International
  Conference on Machine Learning}, pp.~1929--1938, 2018.

\bibitem{bonawitz2017practical}
K.~Bonawitz, V.~Ivanov, B.~Kreuter, A.~Marcedone, H.~B. McMahan, S.~Patel,
  D.~Ramage, A.~Segal, and K.~Seth, ``Practical secure aggregation for
  privacy-preserving machine learning,'' in {\em proceedings of the 2017 ACM
  SIGSAC Conference on Computer and Communications Security}, pp.~1175--1191,
  2017.

\bibitem{aono2017privacy}
Y.~Aono, T.~Hayashi, L.~Wang, S.~Moriai, {\em et~al.}, ``Privacy-preserving
  deep learning via additively homomorphic encryption,'' {\em IEEE Transactions
  on Information Forensics and Security}, vol.~13, no.~5, pp.~1333--1345, 2017.

\bibitem{Hoeffding1963}
W.~Hoeffding, ``Probability inequalities for sums of bounded random
  variables,'' {\em Journal of the American Statistical Association}, vol.~58,
  no.~301, pp.~13--30, 1963.

\bibitem{Bolte2013}
J.~Bolte, S.~Sabach, and M.~Teboulle, ``Proximal alternating linearized
  minimization for nonconvex and nonsmooth problems,'' {\em Mathematical
  Programming}, vol.~146, 08 2013.

\bibitem{nesterov2003introductory}
Y.~Nesterov, {\em Introductory lectures on convex optimization: A basic
  course}, vol.~87.
\newblock Springer Science \& Business Media, 2003.

\end{thebibliography}
\bibliographystyle{ieeetr}

\clearpage

\appendix

%---------------------------------------------------
\section{Convergence Issues of Existing Iterative Baselines}
\label{sec:conv_issues}
\vspace{-0.0cm}
%--------------------------------------------------
In Subsection~\ref{ssec:scenario_4}, we compared the performance of the proposed method {\flt} against {\fedsem} \cite{xie2020multi} and \texttt{PCA+kM+HC}, which turned out to be the closest competitors according to the evaluation results of Scenarios 3. In Scenario 4, we proposed to challenge these methods on our newly introduced \emph{Structured Non-IID FEMNIST} dataset with $C = 10$ predefined clusters each having different set of labels. Basically, in Scenario 4, we impose two levels of non-IIDness, quantity skew and structured label distribution skew. The results in Subsection~\ref{ssec:scenario_4} confirmed that both {\fedsem} and \texttt{PCA+kM+HC} cannot offer a competitive performance against {\flt}. For {\fedsem} and any other iterative model comparison based approach, we argued that this is due to tremendous heterogeneity in model space, and thus, considerable increase in complexity of pairwise model comparisons in such large-scale non-IID scenarios. When it comes to \texttt{PCA+kM+HC}, even though it is a one-shot approach (following the line of thought proposed in \cite{Dennis2020hetero}), the client-side linear dimensionality reduction seems to be the bottleneck, even though it still benefits from hierarchical clustering and misses out on manifold learning with \texttt{UMAP} on the sever side. To further elaborate on this, in a relatively smaller setup with $M = 240$ clients, Fig.~\ref{fig:flt_clustering} demonstrates the efficiency of the proposed client relatedness formation process ($\texttt{FCR}$, Algorithm~\ref{alg:formcluster}) in automatically discovering the cluster structure in Scenario 4. The figure (a) shows how {\flt} manages to almost perfectly group the clients into $10$ clusters in the form of a block-diagonal adjacency matrix. On the other hand in (b) and (c), the cluster formation of {\fedsem} and \texttt{PCA+kM+HC} is far from ideal. The iterative process of cluster formation for {\fedsem} is visualized in Fig.~\ref{fig:fedsem_clustering_iters}. As can be seen, the cluster formation process does not improve beyond the communication round $15$ and even gets slightly worse from there on. This is the reason behind the sub-optimal performance of {\fedsem} \citep{xie2020multi} in large-scale structured non-IID condition imposed in Scenario 4.  
\begin{figure*}[t!]
    \begin{subfigure}[t]{.32\textwidth}
    	\centering
    	\includegraphics[trim={3.62cm 1.32cm 2.62cm 1.31cm},clip, width=\textwidth]{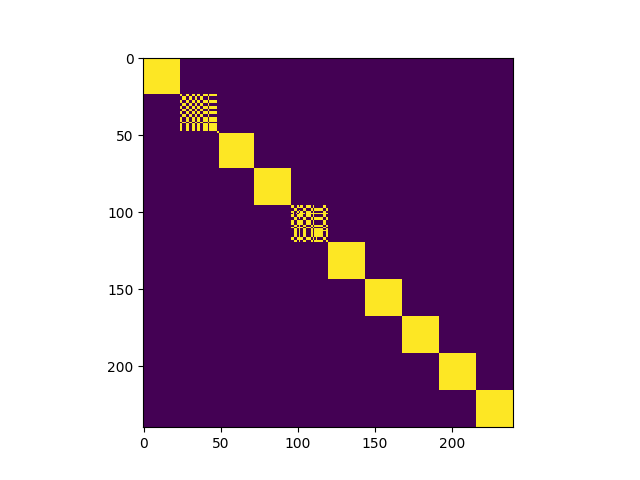}
    	\vspace{-0.3cm}
    	\caption{{\flt}}
    	\label{fig:flt_clustering_1}
    \end{subfigure}
    \hfill
    \begin{subfigure}[t]{.32\textwidth}
    	\centering
	    \includegraphics[trim={3.62cm 1.31cm 2.62cm 1.31cm},clip,width=\textwidth]{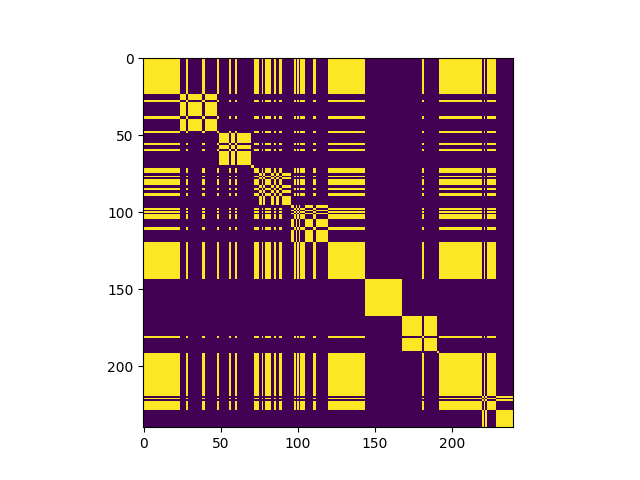}
    	\vspace{-0.3cm}
    	\caption{{\fedsem}}
    	\label{fig:fedsem_clustering}
    \end{subfigure}
    \hfill
    \begin{subfigure}[t]{.32\textwidth}
    	\centering
    	\includegraphics[trim={1.82cm 1.05cm 0.62cm 1.65cm},clip,width=0.965\textwidth]{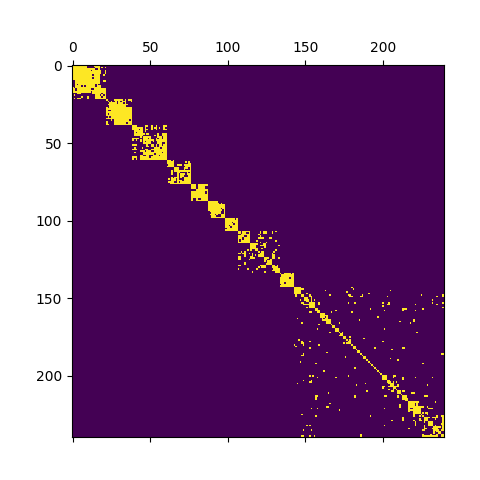}
    	\vspace{0.1cm}
    	\caption{\texttt{PCA+kM+HC}}
    	\label{fig:kmeans_clustering}
    \end{subfigure}
	\vspace{-0.2cm}
	\caption{One-shot cluster formation (re-ordered adjacency matrix corresponding to client relatedness graph) for {\flt} in Scenario 2 with $C = 10$ clusters and a total of $M = 240$ clients. The \emph{structure} of the $C = 10$ clusters is almost perfectly estimated by Algorithm~\ref{alg:formcluster}.}
	\label{fig:flt_clustering}
\end{figure*}
\begin{figure*}[t!]
	\centering
  	\hfill
	\begin{subfigure}[b]{0.18\textwidth}
		\centering
    	\includegraphics[trim={3.62cm 1.31cm 3.62cm 1.31cm},clip,width=\textwidth]{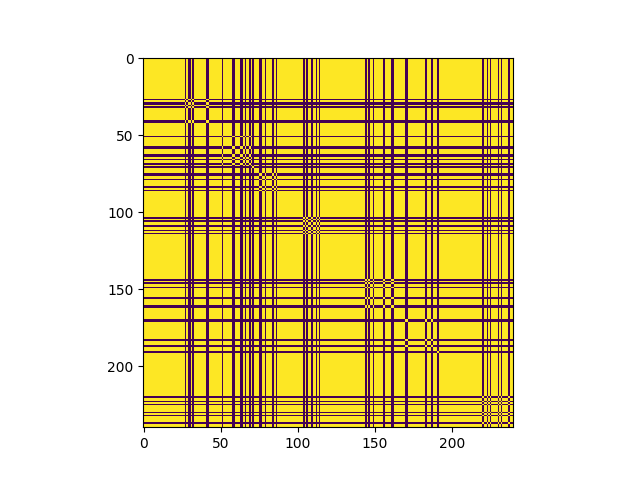}
    	\caption{Comm. round 1}
    \end{subfigure}
	\hfill
	\begin{subfigure}[b]{0.18\textwidth}
		\centering
	    \includegraphics[trim={3.62cm 1.31cm 3.62cm 1.31cm},clip,width=\textwidth]{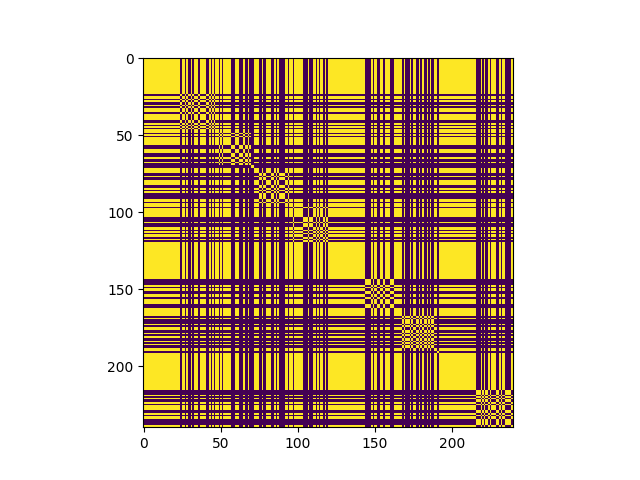}
    	\caption{Comm. round 3}
    \end{subfigure}
	\hfill
    \begin{subfigure}[b]{0.18\textwidth}
    	\centering
	    \includegraphics[trim={3.62cm 1.31cm 3.62cm 1.31cm},clip,width=\textwidth]{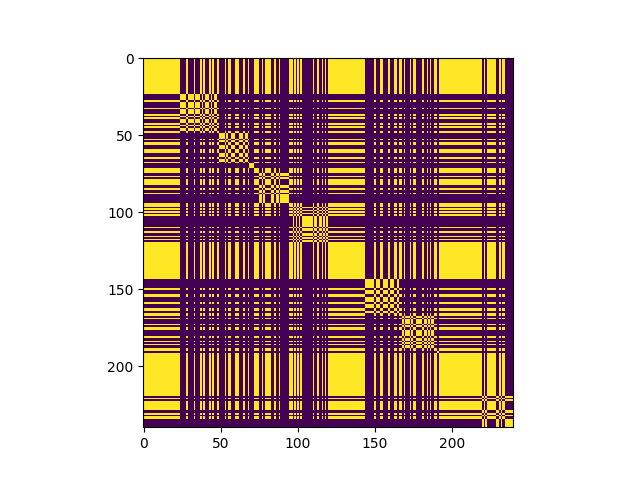}
	    \caption{Comm. round 6}
    \end{subfigure}
    \hfill
    \begin{subfigure}[b]{0.18\textwidth}
    	\centering
	    \includegraphics[trim={3.62cm 1.31cm 3.62cm 1.31cm},clip,width=\textwidth]{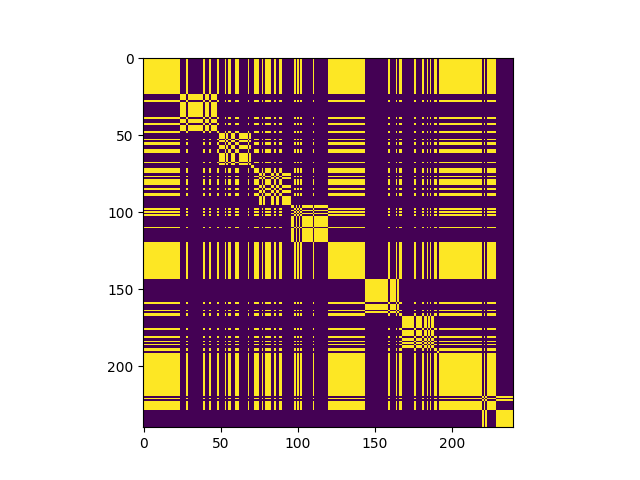}
	    \caption{Comm. round 9}
    \end{subfigure}
    \hfill
    \vfill
    \hfill
    \begin{subfigure}[b]{0.18\textwidth}
    	\centering
	    \includegraphics[trim={3.62cm 1.31cm 3.62cm 1.31cm},clip,width=\textwidth]{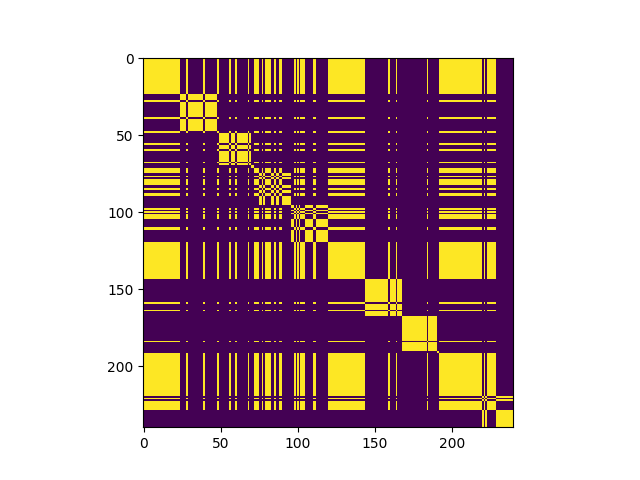}
	    \caption{Comm. round 12}
    \end{subfigure}
    \hfill
    \begin{subfigure}[b]{0.18\textwidth}
    	\centering
	    \includegraphics[trim={3.62cm 1.31cm 3.62cm 1.31cm},clip,width=\textwidth]{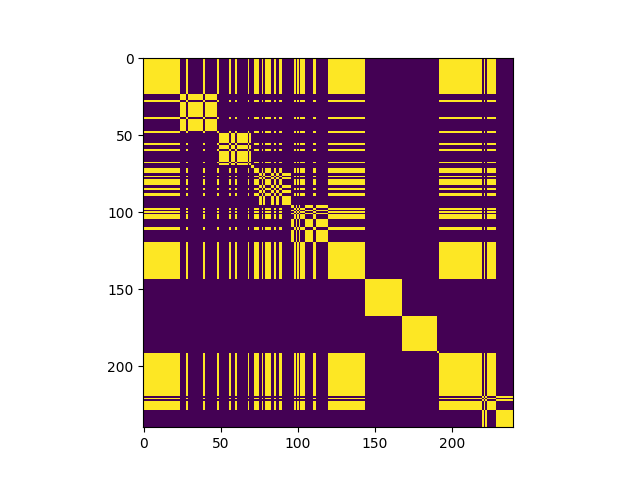}
	    \caption{Comm. round 15}
    \end{subfigure}
    \hfill
    \begin{subfigure}[b]{0.18\textwidth}
    	\centering
	    \includegraphics[trim={3.62cm 1.31cm 3.62cm 1.31cm},clip,width=\textwidth]{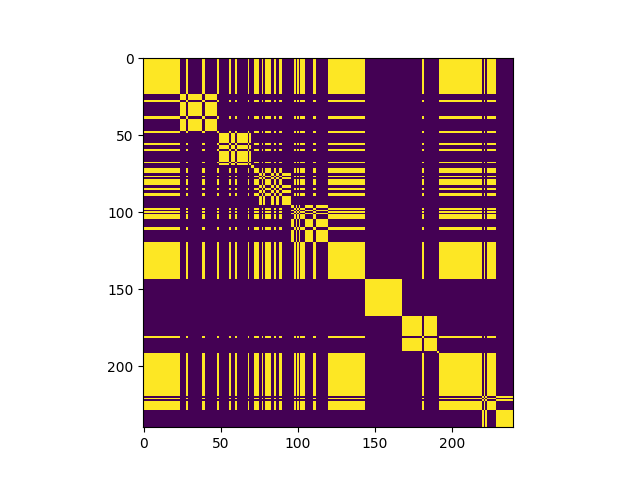}
	    \caption{Comm. round 18}
    \end{subfigure}
    \hfill
    \begin{subfigure}[b]{0.18\textwidth}
    	\centering
	    \includegraphics[trim={3.62cm 1.31cm 3.62cm 1.31cm},clip,width=\textwidth]{./Figures/clust_multicenter_nr_users-240_nr_clusters_10_ep_50_itr_-49.png}
	    \caption{Comm. round 50}
    \end{subfigure}
    % \hfill
    \vspace{-0.2cm}
	\caption{Iterative cluster formation discovery process of {\fedsem} in Scenario 2. The expected outcome is a block-diagonal matrix akin to the one presented in Fig.~\ref{fig:flt_clustering}; however, {\fedsem} gets stuck at comm. round 15.}
	\vspace{-0.0cm}
	\label{fig:fedsem_clustering_iters}
	\vspace{-0.0cm}
\end{figure*}
%

%---------------------------------------------------
\section{Mathematical Support}
\label{sec:math_support}
\vspace{-0.0cm}

%---------------------------------------------------
\subsection{Clustering performance of {\flt}}
\label{ssec:clusteringaccuracy}
\vspace{-0.0cm}
%--------------------------------------------------
In this subsection, we provide an approximate upper-bound on the clustering performance error of {\flt} by considering some common simplifying assumptions. Since our approach is founded upon separating client data in latent space, the overall clustering performance for the entire $M$ clients is determined by the encoder (part of the ConvAE) performance which in turn is affected by several parameters including the number of data samples provided to each client $n_m$, where $\sum_{m \in [M]} n_m = N$, and the total number of target classes (labels) is $L$.

We denote the classification error of the encoder on a single data sample by $P^{\text{cae}}_e$, and the average classification error on (data samples with) label $l$ by $P^l_e$. We then have
\begin{equation} \label{eq:1}
P_{\text{err}}^{\text{cae}} = P(y > \gamma),
\end{equation}
where $\gamma$ is a threshold value, and $y \in \mathbb{R}^e$ is the output of the encoder in the latent domain. Let us denote the total number of samples with label $l$ for client $m$ by $n_m^l$. The client then applies the encoder followed by \texttt{kMEANS}. Assuming that all samples corresponding to each label fall within one of these $k$ clusters, the performance of the autoencoder after averaging over samples corresponding to client $m$, label $l$, and for the $i$-th dimension of the latent embedding can be given by:
\begin{equation} \label{eq:2}
P^{m,l,i}_\text{err} = P(\overline{y}_i > \gamma).
\end{equation}
In this case, $\overline{y}_i$ would be drawn from a Gaussian distribution $\overline{y}_i \sim N({\mu}, {\sigma}^2)$. This essentially means that $\overline{y}$ would be of dimension $e$ and drawn from a Gaussian mixture model.
\begin{theorem}
Clustering error of {\flt} is upper-bounded by:
\begin{align} 
 P^{tot}_{\textup{err}} \lessapprox \sum_{m=1}^{M} \sum_{l=1}^{L} \sum_{i=1}^{e}  \exp\left(-\frac{\left(n_m^l \times Q^{-1}(P_{\textup{err}}^{\text{cae}}) \right)^2}{2}\right).   
\end{align}
where, $Q$ represents the $Q$-function.
\end{theorem}
\begin{proof}
Under the above assumptions, the error in \eqref{eq:2} is given by the following $Q$-function:
\begin{equation} \label{eq:3}
P^{m,l,i}_{\text{err}} = Q\left(n_m^l \times Q^{-1}(P_{\text{err}}^{\text{cae}})\right).
\end{equation}
By employing the Chernoff bound associated with a $Q$-function \citep{Hoeffding1963}, the following bound can be given for \eqref{eq:3}:
\begin{equation} \label{eq:localerror}
P^{m,l,i}_{\text{err}} \leq \exp\left(-\frac{\left(n_m^l \times Q^{-1}(P_{\text{err}}^{\text{cae}}) \right)^2}{2}\right).
\end{equation}
By considering our earlier assumption on the independence of classification tasks conducted by the individual clients, the probability of a correct classification for all the client across all the labels can be given by:
\begin{equation} \label{eq:5}
1-P^{tot}_{\text{err}} = \prod_{m=1}^{M} \prod_{l=1}^{L}\prod_{i=1}^{e}(1-P^{m,l,i}_{\text{err}}).
\end{equation}
Accordingly, the total error can be given by:
\begin{equation} \label{eq:totalerror}
P^{tot}_{\text{err}} \approx \sum_{m=1}^{M} \sum_{l=1}^{L} \sum_{i=1}^{e} P^{m,l,i}_{\text{err}}.
\end{equation}
By inserting \eqref{eq:localerror} in \eqref{eq:totalerror}, we arrive at:
\begin{equation}
 P^{tot}_{\text{err}} \lessapprox \sum_{m=1}^{M} \sum_{l=1}^{L} \sum_{i=1}^{e}  \exp\left(-\frac{\left(n_m^l \times \mathcal{Q}^{-1}(P_{\text{err}}^{\text{cae}}) \right)^2}{2}\right), 
\end{equation}
which provides an upper-bound on the total clustering error of {\flt}.
\end{proof}
\subsection{Convergence of \flt}
\label{ssec:eval}
\vspace{-0.0cm}
%--------------------------------------------------
We prove that under same regularity assumptions of {\fedavg}, {\flt} converges. For analyzing the convergence characteristics of the proposed taskonomy-based federated learning approach, the following assumption are in place:
\begin{assumption}[Lower-boundedness]\label{as:1} All the local cost functions, $F_i$, are lower-bounded.
\begin{equation}
F_i(w_i) \geq \mathcal{\tilde{F}} > -\infty,   \; \; \forall w_i \in \mathbb{R}^d.
\end{equation}
\end{assumption}
\begin{assumption}[Lipschitz continuity]\label{as:2} The function $F$ is $L_W$-smooth.
\begin{equation}
\| F(w_1) - F(w_2) \| \leq  L_{w} \| w_1-w_2 \|,   \; \; \forall w_1, w_2 \in \mathbb{R}^d.
\end{equation}
\end{assumption}

\begin{assumption}[Compactness]\label{as:3}  Set of  local  models $\mathcal{W} =\{w_i | i = 1, 2, ..., M\}$ is compact and convex. 
\begin{equation} \label{eq:6}
   \| \nabla_w F_i(w) - \nabla_w F(w) \|^2 \leq \phi^2, \; \; \| \nabla_w F_i(w) \|^2 \leq \psi^2, 
\end{equation}
for all $w \in \mathbb{W}$, where $\phi$ and $\psi$ are some constants.
\end{assumption}

The {\flt} optimization problem can be written as the following generalized form:
\begin{equation}\vspace{-0.2cm}
\label{generalizedFL}
\min_{w_1, ... w_M \in \mathbb{R}^d} \sum_{m = 1}^{M} p_m F_m(w_i), \; where \;\; w_i = \sum_{j = 1}^{M}\overbar{A}_{i,j}w_j,
\end{equation}
where, $\overbar{A}$ represents the row-normalized connectivity matrix and, has the following property:
\begin{equation}
\sum_{j = 1}^{M} \overbar{A}_{i,j} = 1, \;\; i = 1, 2, ..., M.
\end{equation}
Note that the term on the right-hand-side of the \eqref{generalizedFL} corresponds to the updates on the server side. This is the part that is handled differently for cluster-based algorithms (e.g. {\fedsem}, {\flt} and {\ifca}).
For {\fedavg}, we have $A = \frac{1}{M}\1_{M \times M}$ where $\1_{M \times M}$ is an all-ones matrix.
In addition, the gradient descent update at the client side is given by
\begin{equation}
\label{gradientdescent}
w_i^t=w_i^{t-1} - \eta \nabla_{w_i} F_i(w_i^{t-1}),
\end{equation}
where, $\eta$ is the step size (or learning rate).

The right-hand-side of \eqref{generalizedFL} which is performed at the server side to update the model parameters can be written in a matrix form as follows:
\begin{equation}
\label{centeralaveraging}
W \gets \overbar{A}W.
\end{equation}
Inserting \eqref{centeralaveraging} in \eqref{gradientdescent}, we have
\begin{equation} \label{eq:4}
W^t \gets \overbar{A}(W^{t-1} - \eta \nabla_{W} F(W^{t-1})).
\end{equation}
Here, we work under the \emph{worst case} assumption where only one local epoch is done at the client side, as more local epochs would only help the network converge faster. Under the common regularity assumptions (Assumptions \ref{as:1}, \ref{as:2} and \ref{as:3}), by dropping the convex relation applied to the cost function of \cite{wang2020} as well as simplifying the alternating minimization strategy to a one-step gradient descent, the optimality gap for the stationary solution of \eqref{generalizedFL} is given by:
\begin{equation} \label{eq:8}
G_\mathbb{W}(W^t) \overset{\Delta}{=} \| W^{t} - W^{t-1} \|^2.
\end{equation}
\begin{theorem}[Convergence of optimality gap]
\label{convergencegap}
 Under Assumptions \ref{as:1}, \ref{as:2} and \ref{as:3}, an iterative gradient descent solution for \eqref{generalizedFL} satisfies:
\begin{align} 
\frac{1}{TM}&\sum_{i=1}^M \sum_{t = 1}^{T} \mathbb{E}[G_\mathbb{W}(W^t)] \leq \frac{1}{T \times M}\frac{F(W^1) - F(W^{T})}{(\frac{1}{2\eta} - \frac{L_{W}}{2} - \frac{\eta}{2}L_{W}^2 \| I_M-\overbar{A} \|^2)},
\end{align}
where, $I_M$ represents the $M \times M$ identity matrix.
\end{theorem}
\vspace{-0.4cm}
\begin{proof}
Since $F$ is $L_W$-smooth under Assumption \ref{as:2}, using Lemma 1 in \citep{Bolte2013}, the following holds:
\begin{align}  \label{eq:18}
F(W^t) & \leq F(W^{t-1}) + \langle \nabla_{W} F(W^{t-1}), W^{t}-W^{t-1} \rangle + \frac{L_{W}}{2} \| W^{t}-W^{t-1} \|^2.
\end{align}
\vspace{-0.1cm}

The first term of \eqref{eq:18} is bounded as given by the following lemma:
\begin{lemma}\label{lemma1}
For any $t \geq 0$, the following inequality holds:
\begin{align} \label{eq:lemma4}
\langle &\nabla_{W} F(W^{t-1}), W^{t}-W^{t-1} \rangle \leq -\frac{1}{\eta} \| W^{t}-W^{t-1} \| + \langle \nabla_{W} F(W^{t-1}) - \overbar{A}\nabla_{W}F(W^{t-1}), W^{t}-W^{t-1} \rangle.
\end{align}
\end{lemma}
The proof of the above lemma is given in the following:
\begin{proof}
Based on the update rule sketched in \eqref{gradientdescent}, and employing the right-hand-side term of \eqref{generalizedFL}, we have:
\begin{equation} \label{eq:15}
-\frac{w_i^t-w_i^{t-1}}{\eta} = \sum_{j=1}^{M}\overbar{A}_{i,j} \nabla_w F_j(w_j^{t-1}).
\end{equation}
which leads to:
\begin{align} \label{eq:16}
\langle \nabla_w F(W^{t-1})& -\overbar{A} \nabla_w F(W^{t-1}), W^{t}-W^{t-1} \rangle \nonumber \\ = \langle \nabla_w & F(W^{t-1})+(W^t-W^{t-1})/\eta, W^{t}-W^{t-1} \rangle  \nonumber \\ & \geq \langle \nabla_W F(W^{t-1}), W^{t}-W^{t-1} \rangle \nonumber \\ & + \frac{1}{\eta} \langle W^{t}-W^{t-1}, W^{t}-W^{t-1} \rangle,
\end{align}
completing the proof of Lemma \ref{lemma1}.
\end{proof}
\vspace{-0.2cm}
Continuing on \eqref{eq:18}, by inserting \eqref{eq:lemma4} in \eqref{eq:18}  and using $\langle x, y \rangle \leq \frac{1}{2c}\|x\|^2 + \frac{c}{2}\|y\|^2$ with $c = \frac{1}{\eta}$, we have
\begin{align} \label{eq:17}
F(W^t) - & F(W^{t-1}) \leq \left(-\frac{1}{2\eta} + \frac{L_{W}}{2}\right) \| W^{t}-W^{t-1} \|^2 + \frac{\eta}{2} \|  \nabla_W F(W^{t-1})-\overbar{A} \nabla_W F(W^{t-1}) \|^2,
\end{align}
which can be further simplified as:
\begin{align} \label{eq:25}
F(W^t) - & F(W^{t-1}) \leq \left(-\frac{1}{2\eta} + \frac{L_{W}}{2}\right) \| W^{t}-W^{t-1} \|^2 + \frac{\eta}{2}L_{W}^2 \| I_M-\overbar{A} \|^2  \| W^{t}-W^{t-1} \|^2.
\end{align}
Finally, taking summation over $t$ from 1 to $T$ in \eqref{eq:25}, we arrive at:
\begin{align} \label{eq:26}
\left(\frac{1}{2\eta} - \frac{L_{W}}{2} - \frac{\eta}{2}L_{W}^2 \| I-\overbar{A} \|^2\right) \sum_{t=1}^T \| W^{t}-W^{t-1} \|^2 \leq  F(W^1) - F(W^{T}),
\end{align}
which after taking expectation over client models, completes the proof of the theorem.
\end{proof}

Notably, in \eqref{eq:26}, the factor $(\frac{1}{2\eta} - \frac{L_{W}}{2} - \frac{\eta}{2}L_{W}^2 \| I-\overbar{A} \|^2)$ has to be positive. Working this out, results in $\eta$ following:  
\begin{equation} \label{eq:29}
0 < \eta < \frac{2}{L_{W} + L_{W} \sqrt{1 + 4\| I-\overbar{A}. \|^2}}
\end{equation}

In order to go beyond our worst case assumption of single local epoch and incorporate the impact of $E$, we introduce the notation $W^{t, E}$ where $t \in [1,T]$ indexes the $t$-th communication round and $E$ indicates the final local epoch.
Since the function $F$ still remains $L_W$ smooth under Assumption \ref{as:2} after performing $E$ local epochs, it follows from equation \ref{eq:18} that:
\begin{align} \label{eq:1}
F(W^{t,E}) \leq F(W^{t-1,E}) + \langle \nabla_{W} F(W^{t-1,E}), W^{t,E}-W^{t-1,E}\rangle + \frac{L_{W}}{2} \| W^{t,E}-W^{t-1,E} \|^2.
\end{align}
From our Lemma \ref{lemma1}, it also follows that:
\begin{align} \label{eq:2}
\langle \nabla_{W} F(W^{t-1,E}), W^{t,E}-W^{t-1,E} \rangle \leq -\frac{1}{\eta} \| W^{t,E}-W^{t-1,E} \| \nonumber \\ + \langle \nabla_{W} F(W^{t-1,E}) - \overbar{A}\nabla_{W}F(W^{t-1,E}), W^{t,E}-W^{t-1,E} \rangle,
\end{align}
because equation \ref{eq:15} remains valid even after performing $E$ epochs of gradient descent locally as the operation of aggregating weights and gradients still remains the same as:
\begin{align} \label{eq:3}
-\frac{w_i^{t,E}-w_i^{t-1,E}}{\eta} = \sum_{j=1}^{M}\overbar{A}_{i,j} \nabla_w F_j(w_j^{t-1,E}).
\end{align}
Therefore, the rest of the proof for Lemma 1 remains unaffected.
Plugging equation \ref{eq:2} into equation \ref{eq:1} and applying the Cauchy-Schwartz inequality (mentioned earlier) results in:
\begin{align} \label{eq:4}
F(W^{t,E}) - & F(W^{t-1,E}) \leq \left(-\frac{1}{2\eta} + \frac{L_{W}}{2}\right) \| W^{t,E}-W^{t-1,E} \|^2 + \frac{\eta}{2} \|  \nabla_W F(W^{t-1,E})-\overbar{A} \nabla_W F(W^{t-1,E}) \|^2.
\end{align}
By summing over all communication rounds $[1,T]$ and taking an expectation over all clients $M$ in the same manner as in the case of \eqref{eq:26} we arrive at:
\begin{align} \label{eq:5}
\frac{1}{TM}&\sum_{i=1}^M \sum_{t = 1}^{T} \mathbb{E}[G_\mathbb{W}(W^{t,E})] \leq \frac{1}{TM}\frac{F(W^{1,E}) - F(W^{T,E})}{(\frac{1}{2\eta} - \frac{L_{W}}{2} - \frac{\eta}{2}L_{W}^2 \| I_M-\overbar{A} \|^2)}.
\end{align}
Now, since $F_j(W^{t,E})$ (of client $j$) is a convex and differentiable function and its gradient $\nabla_w F_j(w_j^{t,E})$ is Lipschitz continuous (Assumption \ref{as:2}), we have a convergence bound (as derived in \cite{nesterov2003introductory}) on the gradient descent procedure performed at every client for a given communication round:
\begin{equation}
\label{eq:6}
F_j(w_j^{t,E}) - F_j(w_j^{t,*}) \leq \frac{\| w_j^{t,0} - w_j^{t,*} \|^2}{2 \eta E},
\end{equation}
where $w_j^{t,0}$ and $w_j^{t,*}$ are the $0$-th epoch parameter vector and the optimal parameter vector, respectively. $F_j(w_j^{t,*})$ indicates the optimal/minimal loss function value. Since the parameters are aggregated by multiplying the parameter matrix $W^{t,E}$ with the constant adjacency matrix $\overbar{A}$ and other normalization constants, the convergence bound remains the same in proportion with the parameters, and thus:
\begin{equation}
F(W^{t,E}) - F(W^{t,*}) \leq \frac{\| W^{t,0} - W^{t,*} \|^2}{2 \eta E}.
\end{equation}
By rearranging the numerator of the R.H.S. of \eqref{eq:5}, we obtain:
\begin{align} \label{eq:8}
F(W^{1,E}) - F(W^{T,E}) = F(W^{1,E}) + F(W^{1,*}) - F(W^{1,*}) - F(W^{T,E}) + F(W^{T,*}) - F(W^{T,*}) \nonumber \\ \leq \frac{\| W^{1,0} - W^{1,*} \|^2 + \| W^{t,0} - W^{t,*} \|^2 + 2 \eta E (F(W^{1,*}) - F(W^{T,*}))}{2 \eta E}.
\end{align}

Finally, by using \eqref{eq:8} in the \eqref{eq:5}, we obtain a more elaborate convergence bound on the optimality gap of {\flt} now incorporating $E$:
\begin{align} \label{eq:9}
\frac{1}{TM}&\sum_{i=1}^M \sum_{t = 1}^{T} \mathbb{E}[G_\mathbb{W}(W^{t,E})] \leq \frac{\| W^{1,0} - W^{1,*} \|^2 + \| W^{t,0} - W^{t,*} \|^2 + 2 \eta E (F(W^{1,*}) - F(W^{T,*}))}{2 \eta E TM(\frac{1}{2\eta} - \frac{L_{W}}{2} - \frac{\eta}{2}L_{W}^2 \| I_M-\overbar{A} \|^2)}.
\end{align}
%
% \begin{table*}
% \centering
% \normalsize
% \begin{minipage}{\textwidth}
% \textcolor{blue}{
% \begin{align} \label{eq:8}
% F(W^{1,E}) - F(W^{T,E}) = F(W^{1,E}) + F(W^{1,*}) - F(W^{1,*}) - F(W^{T,E}) + F(W^{T,*}) - F(W^{T,*}) \nonumber \\ \leq \frac{\| W^{1,0} - W^{1,*} \|^2 + \| W^{t,0} - W^{t,*} \|^2 + 2 \eta E (F(W^{1,*}) - F(W^{T,*}))}{2 \eta E}.
% \end{align}
% %
% \begin{align} \label{eq:9}
% \frac{1}{TM}&\sum_{i=1}^M \sum_{t = 1}^{T} \mathbb{E}[G_\mathbb{W}(W^{t,E})] \leq \frac{\| W^{1,0} - W^{1,*} \|^2 + \| W^{t,0} - W^{t,*} \|^2 + 2 \eta E (F(W^{1,*}) - F(W^{T,*}))}{2 \eta E TM(\frac{1}{2\eta} - \frac{L_{W}}{2} - \frac{\eta}{2}L_{W}^2 \| I_M-\overbar{A} \|^2)}.
% \end{align}
% }
% \hrule
% \end{minipage}
% \end{table*}

\clearpage
%---------------------------------------------------
\section{More Detailed Evaluation Results}
\label{sec:more_eval}
\vspace{-0.0cm}
%--------------------------------------------------
%
\begin{table*}[t!]
	\footnotesize
	\caption{Train and test accuracies ($\% \pm$ std. error) for Scenario 1.}
	\vspace{-0.2cm}
	\label{tb:scenario_1_sup}
	\centering
	{\tabcolsep=0pt\def\arraystretch{1.0}
		\begin{tabularx}{460pt}{l *8{>{\Centering}X}}
			\toprule
			& \multicolumn{4}{c}{MNIST}   & \multicolumn{4}{c}{CIFAR10} \tabularnewline \cmidrule(l){2-5}\cmidrule(l){6-9}
			& \multicolumn{2}{c}{train}   & \multicolumn{2}{c}{test}
			& \multicolumn{2}{c}{train}   & \multicolumn{2}{c}{test}
			\tabularnewline \cmidrule(l){2-3}\cmidrule(l){4-5}
			\cmidrule(l){6-7}\cmidrule(l){8-9}
			Method								& acc. 							&var. 	&acc.  	& var. & acc. 							&var. 	&acc.  	& var. \tabularnewline 
			\midrule
			{\fedavg}  					&82.50\scriptsize{$\pm$0.16} &43.19
			&82.73\scriptsize{$\pm$0.38} 	&44.65
			&35.08\scriptsize{$\pm$0.21} &152.54 
			&33.29\scriptsize{$\pm$0.47} 		&108.20 \tabularnewline
			\texttt{Local}  	&99.25\scriptsize{$\pm$0.04}  &0.60
			&97.41\scriptsize{$\pm$0.16}  	&3.44 	
			&99.87    \scriptsize{$\pm$0.02} & 0.80
			&79.81\scriptsize{$\pm$0.40} 	&115.35 \tabularnewline
			\texttt{PCA+kM+HC}  			&83.65\scriptsize{$\pm$0.16} &49.2
			&83.82\scriptsize{$\pm$0.37}  	&41.13 		
			&78.07\scriptsize{$\pm$0.19}     &589.65 
			&75.35\scriptsize{$\pm$0.40} 	&589.6 \tabularnewline
			{\fedsem} \cite{xie2020multi} 					&97.94\scriptsize{$\pm$0.06 } &2.58
			&98.01\scriptsize{$\pm$0.14} 	&2.12
			&83.43    \scriptsize{$\pm$0.17} &725.99
			&78.23\scriptsize{$\pm$0.41} 		&695.01 \tabularnewline
			\rowcolor{LightCyan}
			{\flt} (Enc1)  				&97.96\scriptsize{$\pm$0.06} &2.50 
			&97.98\scriptsize{$\pm$0.14} 	&2.17 	
			&93.52\scriptsize{$\pm$0.11} &36.21
			&87.11\scriptsize{$\pm$0.33} 		&88.80 \tabularnewline
			\rowcolor{LightCyan}
			{\flt} (Enc2) 				&97.95\scriptsize{$\pm$0.06} &2.50
			&97.97\scriptsize{$\pm$0.14} 	&2.17 	
			&93.51 \scriptsize{$\pm$0.11} &39.42 
			&87.17\scriptsize{$\pm$0.33}  		&84.62 \tabularnewline
			\bottomrule
	\end{tabularx}}
\vspace{-.3cm}
\end{table*} 
\begin{table*}[t!]
	\footnotesize
	\caption{Train and test accuracies ($\% \pm$ std. error) for Scenario 2.}
	\vspace{-0.2cm}
	\label{tb:scenario_2_sup}
	\centering
	{\tabcolsep=0pt\def\arraystretch{1.0}
		\begin{tabularx}{460pt}{l *8{>{\Centering}X}}
			\toprule
			& \multicolumn{4}{c}{MNIST}   & \multicolumn{4}{c}{CIFAR10} \tabularnewline \cmidrule(l){2-5}\cmidrule(l){6-9}
			& \multicolumn{2}{c}{train}   & \multicolumn{2}{c}{test}
			& \multicolumn{2}{c}{train}   & \multicolumn{2}{c}{test}
			\tabularnewline \cmidrule(l){2-3}\cmidrule(l){4-5}
			\cmidrule(l){6-7}\cmidrule(l){8-9}
			Method								& acc. 							&var. 	&acc.  	& var. & acc. 							&var. 	&acc.  	& var. \tabularnewline 
			\midrule
			{\fedavg}  					&85.62\scriptsize{$\pm$0.14}    &20.57
			&86.08\scriptsize{$\pm$0.35}    &17.55      &55.87\scriptsize{$\pm$0.22}  &14.18
             &52.77\scriptsize{$\pm$0.50}  &19.18 \tabularnewline
			\texttt{Local}      &99.17\scriptsize{$\pm$0.04}    &0.36
					&96.58\scriptsize{$\pm$0.17}    &1.81 
					&100.0\scriptsize{$\pm$0.0} &0.0
					&70.44\scriptsize{$\pm$0.46}   &72.29 \tabularnewline
			\texttt{PCA+kM+HC}  				&85.27\scriptsize{$\pm$0.15} &19.09 &85.45\scriptsize{$\pm$0.35}    &18.40
			&68.75\scriptsize{$\pm$0.21} &398.82
			&61.75\scriptsize{$\pm$0.49}   &312.22 \tabularnewline
			{\fedsem} \cite{xie2020multi}					&94.26\scriptsize{$\pm$0.10}    &1.27   	&94.14\scriptsize{$\pm$0.23}    &1.20 
			&84.30\scriptsize{$\pm$0.16} &426.27
			&73.26\scriptsize{$\pm$0.44}   &339.39 \tabularnewline
			\rowcolor{LightCyan}

			{\flt} (Enc1)  			&97.54\scriptsize{$\pm$0.06} &0.93
			&97.37\scriptsize{$\pm$0.16} &1.07  
			&95.01\scriptsize{$\pm$0.10} &124.07
			&79.97\scriptsize{$\pm$0.40} &146.84 \tabularnewline
			\rowcolor{LightCyan}

			{\flt} (Enc2) 				&97.52\scriptsize{$\pm$0.06} &0.93
			&97.37\scriptsize{$\pm$0.16} &1.07 
			&94.87\scriptsize{$\pm$0.10} &24.07
			&80.00\scriptsize{$\pm$0.40} &78.66 \tabularnewline
			\bottomrule
	\end{tabularx}}
\vspace{-.3cm}
\end{table*} 
\begin{table*}[t!]
	\footnotesize
	\caption{Train and test accuracies ($\% \pm$ std. error) for Scenario 3.}
	\vspace{-0.2cm}
	\label{tb:scenario_3_sup}
	\centering
	{\tabcolsep=0pt\def\arraystretch{1.0}
		\begin{tabularx}{460pt}{l *8{>{\Centering}X}}
			\toprule
			& \multicolumn{4}{c}{MLP ($T=1000$)}   & \multicolumn{4}{c}{CNN ($T=100$ except for {\ifca} with $T=1500$)} \tabularnewline \cmidrule(l){2-5}\cmidrule(l){6-9}
			& \multicolumn{2}{c}{train}   & \multicolumn{2}{c}{test}
			& \multicolumn{2}{c}{train}   & \multicolumn{2}{c}{test}
			\tabularnewline \cmidrule(l){2-3}\cmidrule(l){4-5}
			\cmidrule(l){6-7}\cmidrule(l){8-9}
			Method								& acc. 							&var. 	&acc.  	& var. & acc. 							&var. 	&acc.  	& var. \tabularnewline 
			\midrule
			{\fedavg}  				&  74.43\scriptsize{$\pm$0.25} & 100
			&  72.76\scriptsize{$\pm$0.76} & 202.61 
			&82.96\scriptsize{$\pm$0.22}  &81.28
			&81.64\scriptsize{$\pm$0.66}  &147.03 \tabularnewline

            % adjusted values mlp? cnn OK
% 			{\texttt{PCA+kM+HC}}   	&74.51\scriptsize{$\pm$0.25} & 104.65 &73.14\scriptsize{$\pm$0.76} &205.41 		&86.40\scriptsize{$\pm$0.2} &74.18 &79.50\scriptsize{$\pm$0.69} &166.60	\tabularnewline
			
			{\texttt{PCA+kM+HC}\tiny{($C{=}3$)}}		&75.77\scriptsize{$\pm$0.25} & 111.44 &71.94\scriptsize{$\pm$0.77} &194.75		  	&83.99\scriptsize{$\pm$0.21} & 89.92 &80.07\scriptsize{$\pm$0.68} &160.32 			\tabularnewline
			
			{\ifca}	\cite{NEURIPS2020_e32cc80b}				&61.57\scriptsize{$\pm$0.28} &48.70
			&  61.24\scriptsize{$\pm$0.84} & 176.38 
			& 84.03\scriptsize{$\pm$0.21}  &35.64
			& 81.47\scriptsize{$\pm$0.66}  &118.71\tabularnewline
			
			{\fedsem}\cite{xie2020multi}				&  77.03\scriptsize{$\pm$0.24} & 100.18 
			&  72.45\scriptsize{$\pm$0.76} & 185.96 
			& 86.57\scriptsize{$\pm$0.20}  &56.83
			& 79.99\scriptsize{$\pm$0.68}  &156.27\tabularnewline
			
			\rowcolor{LightCyan}
			{\flt}		 			&  76.43\scriptsize{$\pm$0.25} & 81.13
			&  74.11\scriptsize{$\pm$0.74} & 171.31 
			&84.93\scriptsize{$\pm$0.21} &56.22
			& 82.14\scriptsize{$\pm$0.65}  &15.03\tabularnewline
			
			\rowcolor{LightCyan}			
			{\flt\tiny{($C{=}2$)}}  	&75.23\scriptsize{$\pm$0.22} &96.79  &72.83\scriptsize{$\pm$0.76} &175.55 		&83.15\scriptsize{$\pm$0.22} &78.27  &80.53\scriptsize{$\pm$0.68} &150.26	\tabularnewline
			
			\rowcolor{LightCyan}			
			{\flt\tiny{($C{=}3$)}}   	&75.73\scriptsize{$\pm$0.25} & 101.14 &72.61\scriptsize{$\pm$0.76} &184.59 		&83.88\scriptsize{$\pm$0.21} & 74.82 &79.72\scriptsize{$\pm$0.69} &170.10	\tabularnewline
			
			\rowcolor{LightCyan}			
			{\flt\tiny{($C{=}5$)}}   	&77.24\scriptsize{$\pm$0.24} & 82.09 &72.05\scriptsize{$\pm$0.77} &180.32 		&85.49\scriptsize{$\pm$0.20} & 63.87 &78.71\scriptsize{$\pm$0.70} &169.22	\tabularnewline
			\bottomrule
	\end{tabularx}}
\vspace{-.3cm}
\end{table*} 
\begin{table*}[t!]
	\footnotesize
	\caption{Train and test accuracies ($\% \pm$ std. error) for Scenario 4.}
	\vspace{-0.2cm}
	\label{tb:scenario_4_sup}
	\centering
	{\tabcolsep=0pt\def\arraystretch{1.0}
		\begin{tabularx}{460pt}{l *8{>{\Centering}X}}
			\toprule
			& \multicolumn{4}{c}{MLP}   & \multicolumn{4}{c}{CNN} \tabularnewline \cmidrule(l){2-5}\cmidrule(l){6-9}
			& \multicolumn{2}{c}{train}   & \multicolumn{2}{c}{test}
			& \multicolumn{2}{c}{train}   & \multicolumn{2}{c}{test}
			\tabularnewline \cmidrule(l){2-3}\cmidrule(l){4-5}
			\cmidrule(l){6-7}\cmidrule(l){8-9}
			Method								& acc. 							&var. 	&acc.  	& var. & acc. 							&var. 	&acc.  	& var. \tabularnewline 
			\midrule
			{\fedavg}  				&45.48\scriptsize{$\pm$0.15} &281.95
			&  46.50\scriptsize{$\pm$0.36} & 100.27 &56.19\scriptsize{$\pm$0.15} &908.67 &53.58\scriptsize{$\pm$0.36} &883.47\tabularnewline
			%\texttt{Local}  				&  71.67\scriptsize{$\pm$0.33} & 131.97 \tabularnewline
			
            % ADJUSTED THR
            % {\texttt{PCA+kM+HC} adj. thr.}				  	&&72.66\scriptsize{$\pm$0.13} & 770.84 &48.75\scriptsize{$\pm$0.36} &467.92 &-\scriptsize{$\pm$-} & -
            %&61.20\scriptsize{$\pm$-} &895.30			\tabularnewline

			{\texttt{PCA+kM+HC}}  	&72.66\scriptsize{$\pm$0.13} & 770.84 &48.75\scriptsize{$\pm$0.36} &467.92 		&62.92\scriptsize{$\pm$0.14} & 863.94 &59.87\scriptsize{$\pm$0.36} &912.39	\tabularnewline
            
            % {\texttt{PCA+kM+HC}\tiny{($C{=}10$)}}  	&-\scriptsize{$\pm$-} & - &-\scriptsize{$\pm$-} &- 		&57.37\scriptsize{$\pm$0.15} &736.40 &54.77\scriptsize{$\pm$0.36} &724.24 			\tabularnewline
            
			{\fedsem} \cite{xie2020multi}				  	&45.19\scriptsize{$\pm$0.15} &508.13
			&  43.53\scriptsize{$\pm$0.36} & 406.60 &79.45\scriptsize{$\pm$0.12} &1295.72 &76.22\scriptsize{$\pm$0.29} &1264.38\tabularnewline
			\rowcolor{LightCyan}
			{\flt} 		 			&91.23\scriptsize{$\pm$0.08} &155.36
			&86.51\scriptsize{$\pm$0.24} & 207.60 &97.95\scriptsize{$\pm$0.04} &6.59 &93.69\scriptsize{$\pm$0.10} &98.65 \tabularnewline
			\bottomrule
	\end{tabularx}}
 \vspace{.0in}
\end{table*} 
Tables~\ref{tb:scenario_1_sup}, \ref{tb:scenario_2_sup}, \ref{tb:scenario_3_sup}, and \ref{tb:scenario_4_sup} respectively summarize the statistics of our experimentation on Scenarios 1, 2, 3 and 4, described in Section~\ref{sec:eval}. Here, we also report the \emph{train} statistics for the sake of completeness. Note that the train accuracies can be extremely high, due to overfitting to the small number of samples clients can possibly have in some scenarios. To further clarify, in Table~\ref{tb:scenario_3_sup} (extended version of Table~\ref{tb:scenario_3} on FEMNIST dataset) for the case of MLP, all the methods are run for $T=1000$ communication rounds because the learning process is slower compared to the case of CNN. As can be seen, after $T = 1000$ communication rounds, the network hardly reaches the same level of performance it achieves with only $T = 100$ rounds with CNN. The only exception here is {\ifca} \citep{NEURIPS2020_e32cc80b} which seems to be slower in convergence in this setting irrespective of the choice of local models. To make it reach a comparable performance for CNN case, we exceptionally allowed {\ifca} to run $T = 1500$ rounds ($15$ times more) in contrast to $T=100$ for the rest of the methods. This demonstrates the difference between the proposed approach ({\flt}) and {\ifca} in term of convergence speed. In this setting, {\flt} is at least $10$ times faster than {\ifca}. The converge graphs corresponding to Tables~\ref{tb:scenario_1_sup}, \ref{tb:scenario_2_sup}, \ref{tb:scenario_3_sup}, and \ref{tb:scenario_4_sup} are respectively illustrated in Figs.~\ref{fig:scenario_1_train_test}, \ref{fig:scenario_2_train_test}, \ref{fig:scenario_3_train_test} and \ref{fig:scenario_4_train_test}. The \emph{train} convergence graphs are plotted with dashed lines and with the same markers as the corresponding \emph{test} curves. Note the different range of communication rounds on the top horizontal axis of Fig.~\ref{fig:scenario_3_train_test} (for CNN on the right) associated with {\ifca}. On the results for Scenario 3, Fig.~\ref{fig:scenario_3_clusters} provides a more complete picture of the impact of imposing a predefined number of clusters in hierarchical clustering (\texttt{HC}) on the performance of {\flt}. As can be seen, the test accuracies slightly degrade compared to {\flt} using the full client relatedness matrix. Increasing the number of clusters seems to have a gradual downgrading impact. 

\begin{figure*}[t!]
\centering
\begin{subfigure}{0.49\textwidth}
	\centering
	\includegraphics[width=0.99\textwidth]{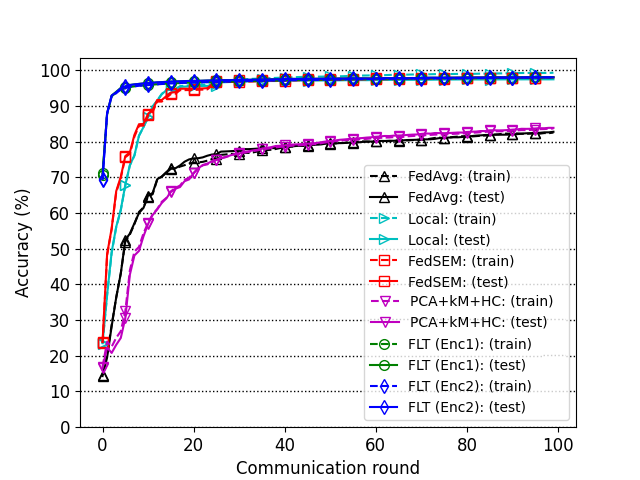}
	\vspace{-0.2cm}
	\caption{Train and test accuracies for MNIST.}
	\label{fig:scenario_1_train_test_mnist}
\end{subfigure}
\hfill
\begin{subfigure}[t!]{0.49\textwidth}
	\centering
	\includegraphics[width=0.99\textwidth]{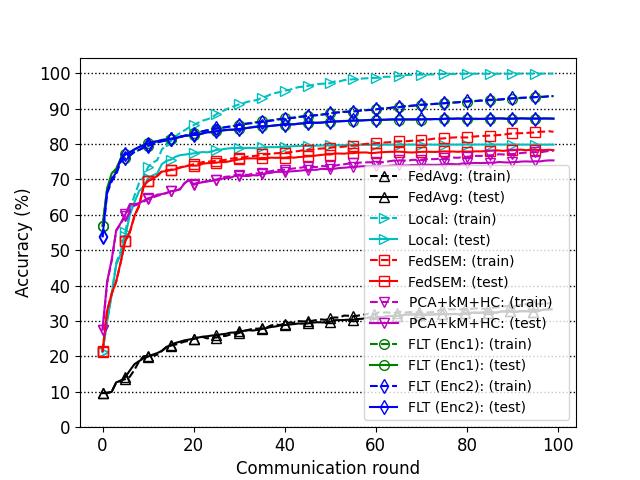}
	\vspace{-0.2cm}
	\caption{Train and test accuracies for CIFAR10.}
	\label{fig:scenario_1_train_test_cifar}
\end{subfigure}
\vspace{-0.2cm}
\caption{Train and test accuracies for Scenario 1, $C = 5$, $M = 100$.}
\label{fig:scenario_1_train_test}
\vspace{-0.3cm}
\end{figure*}
\begin{figure*}[t!]
\centering
\begin{subfigure}{0.49\textwidth}
	\centering
	\includegraphics[width=0.99\textwidth]{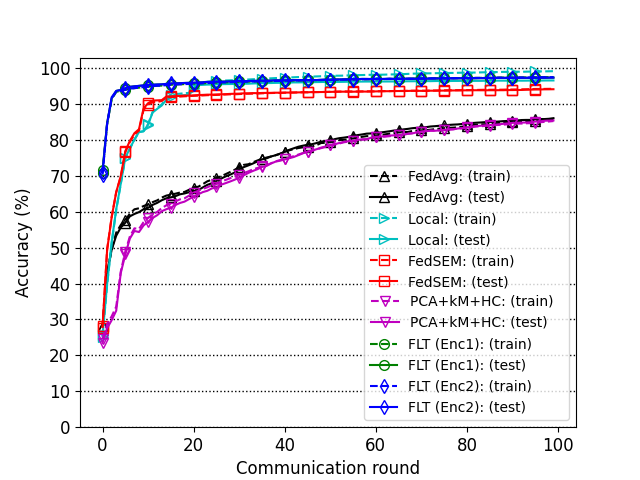}
	\vspace{-0.2cm}
	\caption{Train and test accuracies for MNIST.}
	\label{fig:scenario_2_train_test_mnist}
\end{subfigure}
\hfill
\begin{subfigure}[t!]{0.49\textwidth}
	\centering
	\includegraphics[width=0.99\textwidth]{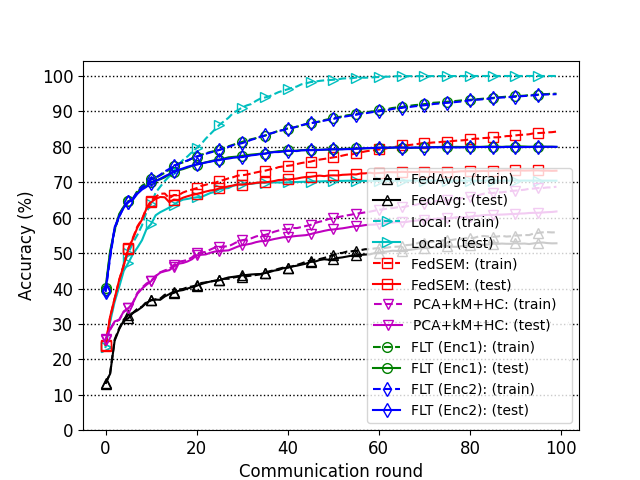}
	\vspace{-0.2cm}
	\caption{Train and test accuracies for CIFAR10.}
	\label{fig:scenario_2_train_test_cifar}
\end{subfigure}
\vspace{-0.2cm}
\caption{Train and test accuracies for Scenario 2, $C = 5$, $M = 100$.}
\label{fig:scenario_2_train_test}
\vspace{-0.3cm}
\end{figure*}
\begin{figure*}[t!]
\centering
\begin{subfigure}{0.49\textwidth}
	\centering
	\includegraphics[width=0.99\textwidth]{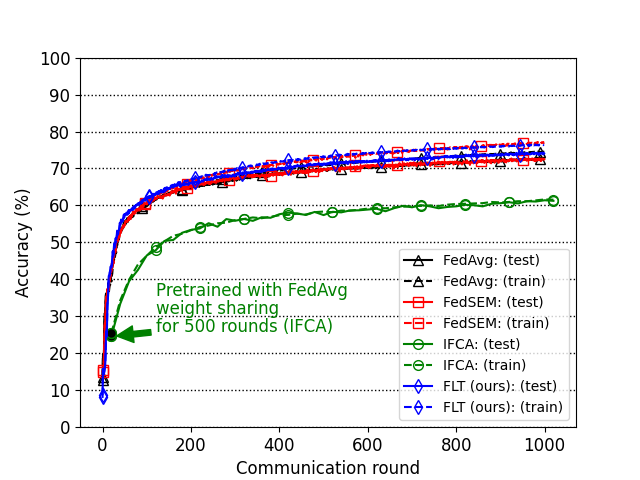}
	\vspace{-0.2cm}
	\caption{Train and test accuracies for MLP.}
	\label{fig:scenario_3_train_test_mlp}
\end{subfigure}
\hfill
\begin{subfigure}[t!]{0.49\textwidth}
	\centering
	\includegraphics[width=0.99\textwidth]{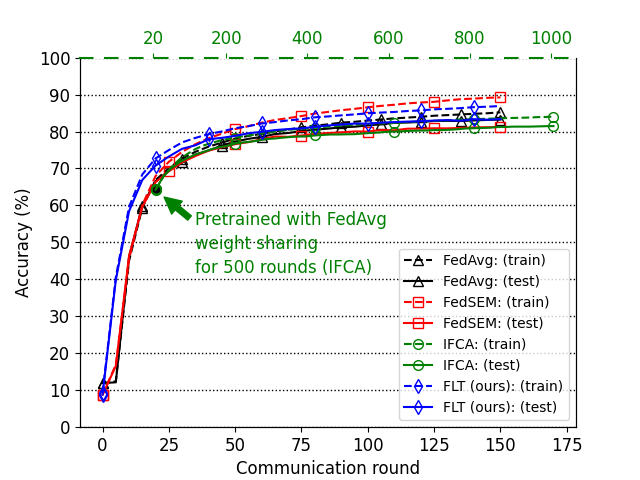}
	\vspace{-0.2cm}
	\caption{Train and test accuracies for CNN.}
	\label{fig:scenario_3_train_test_cnn}
\end{subfigure}
\vspace{-0.2cm}
\caption{Convergence graph of \emph{train} and \emph{test} accuracies for Scenario 3, FEMNIST, $M=200$. On the right for CNN, note the different range of communication rounds on the top horizontal axis associated with {\ifca}.}
\label{fig:scenario_3_train_test}
\vspace{-0.3cm}
\end{figure*}
\begin{figure*}[t!]
    \begin{subfigure}{.50\linewidth}
    	\centering
    	\includegraphics[width=0.90\textwidth]{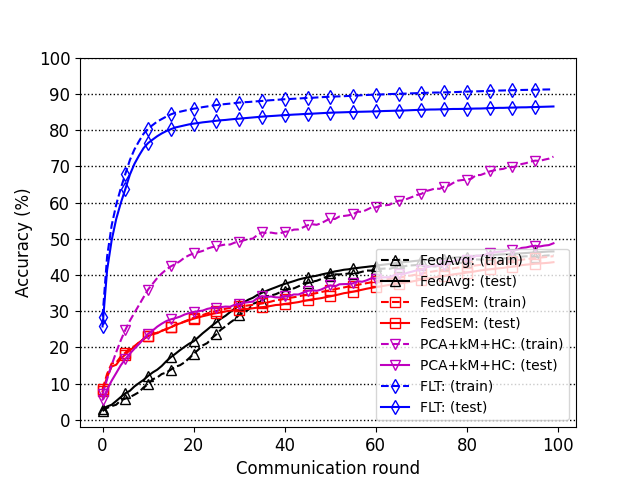}
    	\vspace{-0.2cm}
    	\caption{Train and test accuracies for MLP, $M=2400$.}
    	\label{fig:scenario_4_train_test_mlp}
    \end{subfigure}
    \begin{subfigure}{.50\linewidth}
    	\centering
    	\includegraphics[width=0.90\textwidth]{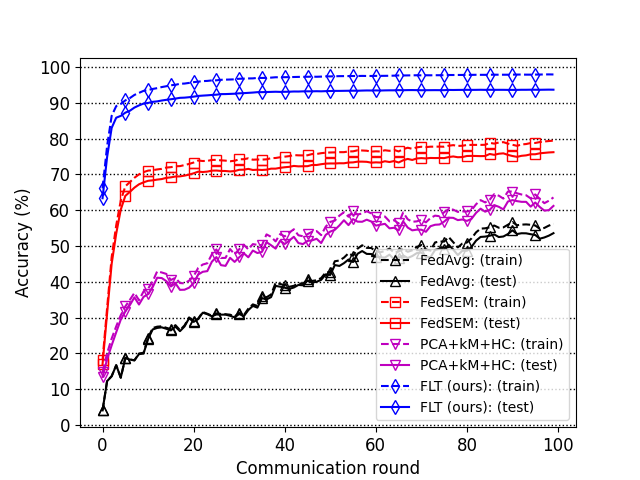}
    	\vspace{-0.2cm}
    	\caption{Train and test accuracies for CNN, $M=600$.}
    	\label{fig:scenario_4_train_test_cnn}
    \end{subfigure}
    \vspace{-0.2cm}
    \caption{Convergence graph of \emph{train} and \emph{test} accuracies for Scenario 4, Structured Non-IID FEMNIST, $C=10$.}
    \label{fig:scenario_4_train_test}
\end{figure*}
\begin{figure*}[t!]
\centering
\begin{subfigure}{0.49\textwidth}
	\centering
	\includegraphics[width=0.99\textwidth]{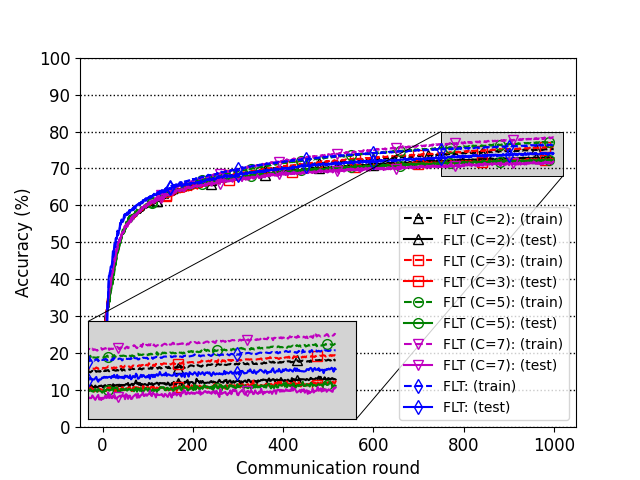}
	\vspace{-0.2cm}
	\caption{Train and test accuracies for MLP.}
	\label{fig:scenario_3_clusters_train_mlp}
\end{subfigure}
\hfill
\begin{subfigure}[t!]{0.49\textwidth}
	\centering
	\includegraphics[width=0.99\textwidth]{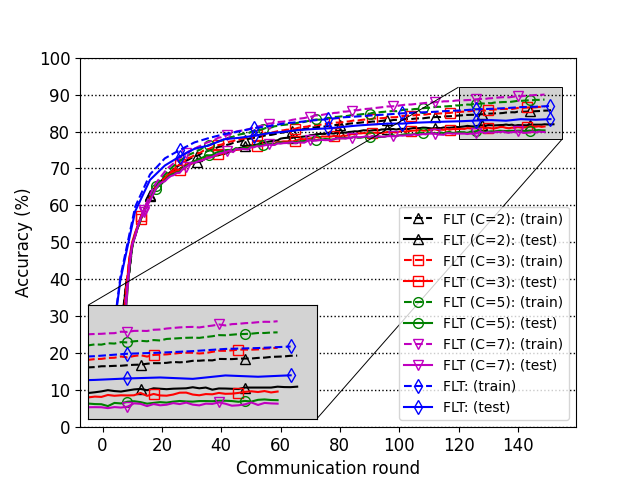}
	\vspace{-0.2cm}
	\caption{Train and test accuracies for CNN.}
	\label{fig:scenario_3_clusters_train_cnn}
\end{subfigure}
\vspace{-0.2cm}
\caption{Convergence graph of \emph{train} and \emph{test} accuracies for Scenario 3, FEMNIST, $M=200$. Comparing the impact of imposing different number of clusters $C$ for hierarchical clustering (\texttt{HC}) on the performance of {\flt}.}
\label{fig:scenario_3_clusters}
\vspace{-0.2cm}
\end{figure*}
%

%---------------------------------------------------
\section{Implementation Details}
\label{sec:implem}
\vspace{-0.0cm}
%--------------------------------------------------

%---------------------------------------------------
\subsection{Structured Non-IID FEMNIST Dataset}
\label{ssec:noniid_femnist}
\vspace{-0.0cm}
%--------------------------------------------------
%
\begin{algorithm}[t!]
	%\SetKwInput{Init}{Memory}
	\SetKwInput{Require}{Require}
	\SetAlgoLined
	\DontPrintSemicolon
	\SetNoFillComment
	%\Init{current best model $g^{*}$, current buffer $\mathcal{B}$}
	\Require{min number of samples per client $\alpha$, power exponent of the power law $\delta$, and cluster formation with 
    	$\mathcal{L} = \{\mathcal{L}_1, \cdots, \mathcal{L}_C\}$,
    	$\mathcal{C} = \{\mathcal{C}_1, \cdots, \mathcal{C}_C\}$. 
	}
    Sort the training data according to the labels assigned to clusters $\mathcal{D}^t = \{\mathcal{D}^t_{\mathcal{L}_1}, \cdots, \mathcal{D}^t_{\mathcal{L}_C}\}$\; 
	 \For{\textup{each cluster} $c \in [C]$}{
	    $\beta^*_c= \argmin_{\beta} (\sum_{m} \alpha + e^{\beta \, m^{\delta}} - |\mathcal{D}^t_{\mathcal{L}_c}|); \forall m \in [M_c]$, with $M_c := |\mathcal{C}_c| $ \;
	    Randomly shuffle samples in $\mathcal{D}^t_{\mathcal{L}_c}$ \;
	    \For{\textup{each client} $m \in [M_c]$ }{
	        Sequentially assign a batch of $\alpha + e^{ \beta^*_c \, m^{\delta} }$ data samples from $\mathcal{D}^t_{\mathcal{L}_c}$  \;
	    }
	}
	\caption{Structured Non-IID FEMNIST Sampler}\label{alg:sampling}
\end{algorithm}
As we explained in Section~\ref{sec:eval}, our newly introduced ``Structured Non-IID FEMNIST" is built upon the ``balanced'' subset of EMNIST dataset \citep{cohen2017emnist}. We impose two levels of non-IIDness: i) quantity skew: meaning that the clients will have a varying number of data samples; ii) structured label skew: where we predefined a set of $C$ clusters and clients in different clusters would have uncorrelated sets of labels. The sampling strategy leading to the Structured Non-IID FEMNIST is summarized in the form of pseudo-code in Algorithm~\ref{alg:sampling}. In the algorithm, $\mathcal{L}_c$ refers to the set of labels clients in cluster $c$ sample their data from. $\mathcal{C}_c$ denotes the client membership set of cluster $c$ containing the client IDs assigned to the cluster with $|\mathcal{C}_c| = M_c$ and $\sum_c M_c =  M$. As such, in this construct, we are grouping clients into a total of $C$ disjoint clusters and the members of each cluster would samples from specific set of labels in $\mathcal{L}_c$. The power law for quantity skew is assumed to follow $\alpha + e^{\beta \, m^\delta}$, where $\alpha$ is a constant denoting the minimum number of samples a client can have and $\delta$ is the power exponent constant. Obviously, any reasonable power law can straightforwardly be accommodated here. The algorithm requires, $\delta$ and $\alpha$, as well as the details of the predefined cluster formation and client membership encoded in $\mathcal{L}$ and $\mathcal{C}$. Firstly, the training data is sorted according to the labels per cluster $\mathcal{D}^t = \{\mathcal{D}^t_{\mathcal{L}_1}, \cdots, \mathcal{D}^t_{\mathcal{L}_C}\}$. Next, each cluster and associated clients are processed. Per cluster, depending on the membership size ($M_c := |\mathcal{C}_c|$) the appropriate $\beta$ is calculated, which will be the same for clusters of the same size. To ensure clients take their samples from almost all the labels allowed in that cluster (denoted by $\mathcal{L}_c$ for cluster $c \in [C]$), a random shuffling is applied to the data samples, and therefore to the associated labels. Finally, within each cluster, a batch of data samples of size $\alpha + e^{\beta^*_c \, m^{\delta}}$ (following the power low) is assigned to client $m$ in cluster $c$. More details and the implementation can be found in our GitHub repository shared earlier.

%---------------------------------------------------
\subsection{Machines}
\label{ssec:machines}
\vspace{-0.4cm}
%--------------------------------------------------
We simulated the federated learning settings on standard virtual machines (VMs) on Microsoft Azure cloud. Our server VM (NC6v3) includes 2 Intel(R) Xeon(R) E5-2690 CPUs (each has 6 vCPUs) and 2 NVidia V100 Tesla GPUs.

%---------------------------------------------------
\subsection{Software}
\label{ssec:software}
\vspace{-0.4cm}
%--------------------------------------------------
The full code and all the associated experiments are publicly available at: \url{https://github.com/hjraad/FLT/}.

%---------------------------------------------------
\subsection{Model architectures}
\label{ssec:model_arch}
\vspace{-0.4cm}
%--------------------------------------------------
For reader's reference, below is detailed model architectures used in our experimentation. The multi-layer perceptron (MLP) employed in local model training is summarized in Table~\ref{tb:mlp_arch}. We have used the CNN architecture proposed in the LEAF package \citep{caldas2018leaf}, as summarized in Table~\ref{tb:cnn_arch}. Finally, in Table~\ref{tb:enc_arch} we have summarized our main convolutional Autoencoder (ConvAE) which is either trained apriori, or will be fine-tuned on client data. After training or fine-tuning, the encoder section of the ConvAE is used for extracting embeddings which will in turn be used for the formation of client relatedness graph in Algorithm~\ref{alg:formcluster}. 

\begin{table*}[h]
\centering 
{\footnotesize
\begin{Verbatim}[samepage=true]
==========================================================================
Layer (type:depth-idx)    Specs
==========================================================================
Linear: 1-1               in_features=flatten(input), out_features=200
Dropout: 1-2              p=0.5
Linear: 1-3               in_features=200, out_features=num_classes
==========================================================================
Total params: 169,462
Trainable params: 169,462
Non-trainable params: 0
==========================================================================
\end{Verbatim}
}
\vspace{-0.5cm}
\caption{MLP architecture}
\label{tb:mlp_arch}
\end{table*}
\begin{table*}[h]
{\footnotesize 
\begin{Verbatim}[samepage=true]
==========================================================================
Layer (type:depth-idx)    Specs
==========================================================================
Conv2d: 1-1               channels=32, kernel_size=5, stride=1, padding=2
MaxPool2d: 1-2            kernel_size=2, stride=2, padding=0
Conv2d: 1-3               channels=63, kernel_size=5, stride=1, padding=2
Linear: 1-4               in_features=flatten(previous), out_features=2048
Linear: 1-5               in_features=2048, out_features=num_classes
==========================================================================
Total params: 3,459,582
Trainable params: 3,459,582
Non-trainable params: 0
==========================================================================
\end{Verbatim}
}
\vspace{-0.5cm}
\caption{CNN architecture}
\label{tb:cnn_arch}
\end{table*}
\begin{table*}[h]
\centering 
{\footnotesize
\begin{Verbatim}[samepage=true]
==========================================================================
Layer (type:depth-idx)    Specs
==========================================================================
Conv2d: 1-1               channels=16, kernel_size=3, stride=1, padding=1
Conv2d: 1-2               channels=4, kernel_size=3, stride=1, padding=1
MaxPool2d: 1-3            kernel_size=2, stride=2, padding=0
Linear: 1-4               in_features=196, out_features=128
Linear: 1-5               in_features=128, out_features=196
ConvTranspose2d: 1-6      channels=16, kernel_size=2, stride=2
ConvTranspose2d: 1-7      channels=16, kernel_size=2, stride=2
==========================================================================
Total params: 51,577
Trainable params: 51,577
Non-trainable params: 0
==========================================================================
\end{Verbatim}
}
\vspace{-0.5cm}
\caption{Encoder (ConvAE) architecture}
\label{tb:enc_arch}
\end{table*}

\end{document}